\documentclass[journal]{article}
\usepackage{lineno}
\usepackage[utf8]{inputenc}
\usepackage{amsmath, amsfonts,amssymb,amsthm}
\usepackage{algorithmicx}
\usepackage{algorithm}
\usepackage{algpseudocode}
\usepackage{booktabs}
\usepackage{graphics}
\usepackage{mathtools}
\usepackage{xcolor}
\usepackage{comment}
\usepackage{cases}
\usepackage{mathtools}
\usepackage{caption}
\usepackage{enumerate}

\newtheorem{theorem}{Theorem}
\newtheorem{lemma}{Lemma}
\newtheorem{definition}{Definition}

\DeclareMathOperator*{\bin}{Bin}
\DeclareMathOperator*{\prob}{Pr}

\newcommand{\probb}[1]{\Pr\left(#1\right)}
\newcommand{\todo}[1]{}
\newcommand{\N}{\mathbb{N}} 
\newcommand{\R}{\mathbb{R}} 
\newcommand{\ep}{\varepsilon}
\newcommand{\sspace}{\mathcal{X}} 
\newcommand{\extsspace}{\mathcal{Y}} 

\usepackage{xspace}
\newcommand{\substring}{\ensuremath{\textsc{SubString}_k}\xspace}
\newcommand{\leadingones}{\ensuremath{\textsc{LeadingOnes}_k}\xspace}

\newcommand{\pmut}{\ensuremath{p_{\text{mut}}}}

\newcommand{\A}[1]{A_{\geq #1}}
\newcommand{\expect}[1]{E\left[#1\right]}

\newcommand{\pinc}{\ensuremath{p_{\text{inc}}}}
\newcommand{\zz}{\tfrac{1-\zeta}{\alpha_0}}

\author{
  Brendan Case\\
  School of Computer Science\\and the Complex Systems Center\\
  University of Vermont, USA\\
  \tt{Brendan.Case@uvm.edu}\\
  \and
  Per~Kristian~Lehre\\
  School of Computer Science\\
  University of Birmingham, UK\\
  \tt{P.K.Lehre@cs.bham.ac.uk}
}

\begin{document}
\title{Self-adaptation in non-Elitist Evolutionary Algorithms on
   Discrete Problems\\with Unknown Structure}

\maketitle

\begin{abstract}
  A key challenge to make effective use of evolutionary algorithms is
  to choose appropriate settings for their parameters. However, the
  appropriate parameter setting generally depends on the structure of
  the optimisation problem, which is often unknown to the
  user. Non-deterministic parameter control mechanisms adjust
  parameters using information obtained from the evolutionary
  process. Self-adaptation -- where parameter settings are encoded in
  the chromosomes of individuals and evolve through mutation and
  crossover -- is a popular parameter control mechanism in evolutionary
  strategies. However, there is little theoretical evidence that
  self-adaptation is effective, and self-adaptation has largely been
  ignored by the discrete evolutionary computation community.

  Here we show through a theoretical runtime analysis that a non-elitist,
  discrete evolutionary algorithm which self-adapts its mutation rate
  not only outperforms EAs which use static mutation rates on
  \leadingones, but also improves
  asymptotically on an EA using a state-of-the-art control mechanism.
  The structure of this problem depends on a parameter $k$,
  which is \emph{a priori} unknown to the algorithm, and which is
  needed to appropriately set a fixed mutation rate. The self-adaptive
  EA achieves the same asymptotic runtime as if this parameter was known to the
  algorithm beforehand, which is an asymptotic speedup for this
  problem compared to all other EAs previously studied.  
  An experimental study of how the mutation-rates evolve show that
  they respond adequately to a diverse range of problem structures.

  These results suggest that self-adaptation should be adopted more
  broadly as a parameter control mechanism in discrete, non-elitist
  evolutionary algorithms.  
\end{abstract}


\section{Introduction}
\label{sec:intro}

Evolutionary algorithms (EAs) have long been heralded for their easy
application to a vast array of real-world problems. In their earlier
years of study, two of the advantages which were often given were their
robustness to different parameter settings, such as mutation rate and
population size, and their effectiveness in domains where little is
known about the problem structure \cite{DeJong2006}. However, progress in the empirical
and theoretical study of EAs has shown many exceptions to these
statements. It is now known that even small changes to the basic parameters
of an EA can drastically increase the runtime on some problems
\cite{bib:Doerr2013, bib:Lehre2013}, and more recently
\cite{lengler2018}, and that hiding some aspects of the problem
structure from an EA can decrease performance \cite{bib:Cathabard2011,
  bib:Doerr2017, bib:Doerr2015,
  bib:Einarsson2018}.

A popular solution to overcoming these shortcomings is \emph{parameter
  tuning}, where the parameters are adjusted between runs of the
algorithm. Since the parameters remain fixed throughout the entire run
of the optimisation process under this scheme, this parameter scheme
is said to be \emph{static} \cite{bib:Eiben1999}. While the majority
of theoretical works have historically investigated static parameter
settings, a weakness of parameter tuning is that effective parameter
settings may depend on the current state of the search process.

An alternative approach is a \emph{dynamic} parameter scheme, which
has long been known to be advantageous compared to static parameter
choices in certain settings \cite{Droste2001, Wegener2005}, reviewed
in \cite{bib:Doerr2018e}. In contrast to parameter tuning, adjusting
parameters in this way is referred to as \emph{parameter control}
\cite{bib:Eiben1999}. Dynamic parameter control changes parameters of
the EA during its execution, usually depending on the EA's state in
the optimisation process or on time. While this can lead to provably
better 
performance, many theoretically-studied algorithms are
\textit{fitness-dependent}, meaning they set parameters according to
the given optimisation function. While important for understanding the
limits of parameter control, such control schemes are often ill-suited
for more general optimisation tasks or on problems where finding an
effective fitness-dependent parameter setting is impractical
\cite{bib:Doerr2018e}.
Thus, practitioners may find it challenging to transfer
theoretical results about fitness-dependent algorithms to an 
applied setting.

A more flexible way to dynamically adjust parameters is to use
feedback from the algorithm's recent performance. This
self-correcting, or \emph{adaptive} approach to parameter control has
been present in Evolutionary Strategies since their beginning with the
$1/5$-th rule; however, results concerning this kind of adjustment
have only recently been seen in the theoretical literature for discrete
EAs \cite{bib:Doerr2015c, Doerr2019, lassig2011}. The advantages of adjusting
parameters \emph{on the fly} in this way include a reduction in design
decisions compared to fitness-dependent algorithms, and the ability
for the same adaptive scheme to work well for a wider range of
optimisation problems \cite{bib:Karafotias2012}.

The adaptive parameter control scheme we consider employs
\emph{self-adaptation} of the EA's mutation rates, where mutation
rates are encoded into the genome of individual solutions.  With
self-adaptation, the mutation rate itself is mutated when an
individual undergoes mutation. As far as we know, within the theory of
discrete EAs there are only two existing studies of
self-adaptation. In \cite{bib:Dang2016}, a self-adapting population
using two mutation rates is shown to have a runtime (expected number
of fitness evaluations) of $O(n\lambda \log \lambda + n^2)$ on a
simple peak function, while the same algorithm using any fixed
mutation rate took $e^{\Omega(n)}$ evaluations with overwhelming
probability. Recently, Doerr et al. gave an example of a $(1,\lambda)$
EA using self-adaptation of mutation rates with expected runtime
$O(n\lambda/\log \lambda + n\log n)$ on \textsc{OneMax} when
$\lambda\geq (\ln n)^{1+\ep}$, an asymptotic speedup from the classic
$(1 + \lambda)$ EA \cite{bib:Doerr2018c}.  However, the former work
optimistically assumes one of the two available mutation rates are
appropriate for the given setting, so that any individual can easily
switch to an ideal mutation rate in a single step, and the latter only
keeps the mutation rate of the best individual after each generation,
which makes tracking the trajectory of mutation rates less difficult
than if there were multiple parents with different mutation rates.
Therefore, while both these algorithms were effective, these two results offer
only a preliminary theoretical understanding of the full range of
self-adaptive mechanisms. Further, the use of limited mutation rates
or a single parent is unrealistic in many real-world settings (i.e.,
where cross-over is frequently used). Thus, it remains an open
question whether a population-based EA can effectively adapt mutation
rate without these assumptions.

We answer this question in the affirmative, introducing an extension
of the $(\mu, \lambda)$ EA which uses self-adaptation of mutation
rates over a continuous interval (Algorithm \ref{algo:stepwise-sa}). In each generation, a new population
of $\lambda$ individuals is created by selecting among the $\mu$
individuals with highest fitness, ties broken according to higher
mutation rate. Each selected individual then multiplies its current
mutation rate by a factor of either $A> 1$ or $b\in(0,1)$, effectively
increasing or decreasing its mutation rate, before undergoing bitwise
mutation. To evaluate the capability of the EA to adapt its mutation rate, we
choose a problem where selecting the right mutation rate is critical,
and where the correct setting can be anywhere between a small constant
to $n/2$, where $n$ is the problem instance size. We show that when optimising $\textsc{LeadingOnes}_k$ the
self-adaptive algorithm has an expected runtime of $O(k^2)$ so long as
$\lambda = O(\log n)$ and $k\geq (\log n)^2$, which is the same
runtime as if $k$ were known. As discussed in more detail in Section
\ref{sec:opt-against-adversary}, this is a significant speedup
compared to an EA using a static choice of mutation rate, which can
only achieve $\Theta(nk)$ on $\textsc{LeadingOnes}_k$. This is also an
asymptotic speedup from the best-known runtime shown in
\cite{bib:Doerr2015}, and indeed is asymptotically optimal among all
unary unbiased black-box algorithms \cite{bib:Badkobeh2014}.

\subsection{Theory of Adaptive Parameter Control}
\label{sec:theory-param-control}

In the following summary of recent results in the theory of 
parameter control in EAs, 
we use the language of
Eiben, Hinterding, and Michalewicz \cite{bib:Eiben1999} to distinguish
between types of parameter control. A parameter control scheme
is called \emph{deterministic} if it uses time or other predefined,
fitness-independent factors to adjust parameters, and \emph{adaptive}
if it changes parameters according to feedback from the optimisation
process. As further distinguished in \cite{bib:Doerr2015c}, a notable
distinction among adaptive algorithms is whether or not they are
\emph{fitness-dependent}, i.e. whether they directly use a particular
fitness function when choosing parameter settings. Adaptive algorithms
which are fitness-independent are either
\emph{self-adjusting}, where a global parameter is modified according
to a simple rule, or \emph{self-adaptive}, where the parameter is
encoded into the genome of an individual and modified through mutation.

For a comprehensive survey
of the theory of parameter control in discrete settings, we refer the
reader to Doerr and Doerr's recent review
\cite{bib:Doerr2018e}. We now highlight some 
themes from the theory of parameter control relevant to this paper.

\emph{Comparison of fitness-independent mechanisms to
  fitness-dependent ones:} Often, the best parameter settings have a
fitness-dependent expression which depends on the precise fitness
value of the search point at that time. While these settings are
typically problem-specific, there is an increasing number of
self-adjusting algorithms which are nearly as efficient, despite not
being tailored to a particular fitness landscape. A common strategy is
to first analyse the fitness-dependent case, followed by a
self-adjusting scheme which attempts to approximate the behaviour of
the fitness-dependent one.  For example, in \cite{doerr2015} a novel
$(1+(\lambda,\lambda))$ GA is shown to need only $\Theta(n)$ fitness
evaluations on \textsc{OneMax} when using a fitness-dependent
offspring size of $\lambda = \Theta(\sqrt{n/(n -
  \textsc{OM}(x))})$. This result is then extended using an adaptive
mechanism based on the $1/5$th rule, where a key element to proving
the algorithm's effectiveness is in demonstrating the adaptive GA's
offspring size $\lambda$ is quickly attracted to within a constant
factor of the fitness-dependent value \cite{Doerr2018}.  A similar
pattern of discovery occurred for the $(1+\lambda)$ EA on
\textsc{OneMax}. Badkobeh et al. first showed that a fitness-dependent
mutation rate led to an expected runtime of
$O\left(\frac{n\lambda}{\log\lambda} + n\log n\right)$, which is
asymptotically tight among all $\lambda$-parallel mutation-based
unbiased black-box algorithms and a speedup from the static-mutation
case \cite{bib:Badkobeh2014}. This was followed by
\cite{bib:Doerr2017b}, where a self-adjusting $(1+\lambda)$ EA is
shown to have the same asymptotic runtime when $\lambda = n^{O(1)}$,
and the aforementioned result for the self-adaptive $(1, \lambda)$ EA
\cite{bib:Doerr2018c}.  Again, it is shown the algorithm is able to
quickly find mutation rates close to the fitness-dependent values. The
mutation rates are shown to stay within this optimal range using
occupation bounds.

For \textsc{LeadingOnes}, it was first demonstrated by B\"ottcher et
al. in \cite{bib:Bottcher2010} that the bitflip probability
$1/(\textsc{LO}(x) + 1)$ led to an improved runtime of roughly
$0.68n^2$ for the $(1+1)$ EA on \textsc{LeadingOnes}.  Since then,
experimental results for the self-adjusting $(1+1)_{\alpha}$ EA
suggest the algorithm is able to closely approximate this value
\cite[Fig. 3]{bib:Doerr2018}, and hyper-parameters for the algorithm
have been found to yield the asymptotically optimal bound
$0.68n^2(1+o(1))$ \cite{Doerr2019}.

\emph{Interplay between the mutation rate and selective pressure in
  non-elitist EAs:} While adaptive parameter control has been studied
considerably less in non-elitist EAs, the critical balance between a
non-elitist EA's mutation rate and selective pressure (how much the
algorithm tends to select the top individuals in the population) takes
on new importance when using self-adaptation of mutation rates. In
\cite{bib:Lehre2012}, the linear-ranking EA is shown to optimise a
class of functions in a sub-exponential number of fitness evaluations
only when the selective pressure is in a narrow interval, proportional
to the mutation rate. A more general result for non-elitist EAs using
mutation rate $\chi/n$ is found in \cite[Corollary 1]{bib:Lehre2010},
where for a variety of selective mechanisms the lower bound
$\chi > \ln(\alpha_0) + \delta$ is given, where $\alpha_0$ is the
reproductive rate and $\delta\in(0,1)$ is a constant (the reproductive
rate is one measure of the selective pressure on an EA, see Definition
\ref{def:reproductive-rate}). If $\chi$ exceeds this bound, with
overwhelming probability any algorithm using this rate will have
exponential runtime on any function with a polynomial number of global
optima.  This negative result is extended in \cite[Theorem
2]{bib:Dang2016} to include non-elitist EAs which choose from a range
of $m$ different mutation rates by selecting mutation rate $\chi_i/n$
with probability $q_i$. Roughly, if
$\sum_{i=1}^mq_ie^{-\chi_i} \leq (1-\delta)/\alpha_0$, the algorithm
will be ineffective.

\subsection{Optimisation Against an Adversary}
\label{sec:opt-against-adversary}

We will analyse the performance of our algorithm on the
$\textsc{LeadingOnes}_k(x)$ function, which counts the number of
leading 1-bits, but only through the first $k$ bits and ignores the
rest of the bitstring:
\begin{definition}
  \label{def:klo-substring}
  For $x\in\{0,1\}^n$, and $1\leq k\leq n$,
  \begin{align*}
    \textsc{LO}_k(x) &:=: \textsc{LeadingOnes}_k(x) := \sum_{i=1}^k\prod_{j=1}^ix_j.
  \end{align*}
\end{definition}

The setting in which we consider this function is referred to as
\emph{optimisation against an adversary}. This can be viewed as an
extension of the traditional black-box optimisation setting, in which
the algorithm does not have access to the problem data or structure
and must learn only through evaluating candidate solutions. Framing
the study of EAs within the context of black-box optimisation, and its
corresponding black-box complexity theory, is of growing interest to
the theoretical community \cite{bib:Doerr2018b}.  Optimisation against
an adversary adds the additional constraint that the value $k$ is also
unavailable to the algorithm. That is, prior to each run of the
optimisation algorithm, an adversary chooses an integer $k\leq n$ and
the algorithm must optimise the resulting problem $f_k$. Effectively,
the adversary is able to choose some $f_k$ from a class of functions
parameterised by $k$, and the algorithm could have to solve any
problem from this class. Note that the adversary is not able to
actually permute any bits during optimisation, they only influence the
optimisation task through their choice of $k$. A similar problem was
first analysed by Cathabard et al. in \cite{bib:Cathabard2011}, along
with an analogous $\textsc{OneMax}_k$ function, though here $k$ was
sampled from a known distribution. The setting where $k$ corresponds
to an unknown initial number of bits which impact fitness has become
known as the \emph{initial segment uncertainty model}. The closely
related \emph{hidden subset problem}, which is analogous to the
initial segment model except the $k$ meaningful bits can be anywhere
in the bitstring, has also been studied for $\textsc{LeadingOnes}_k$
and $\textsc{OneMax}_k$ \cite{bib:Doerr2017, bib:Doerr2015,
  bib:Einarsson2018}. Since our algorithm always flips all bits with
equal probability during mutation, our results immediately extend to
this class of problems. Optimisation against an adversary further
generalises this terminology to contain any problem in which an
adversary can control the hidden problem structure through their
choice of $k$. For example, it includes the $\textsc{SubString}_k$
function introduced in Section \ref{sec:experiments}.

The addition of an adversary can be difficult for EAs with static
mutation rate due to the following phenomenon: consider a $(1+1)$ EA
using constant mutation probability $p$, and suppose we are attempting
to optimise $\textsc{LeadingOnes}_k$ against an adversary. If $k$ is
far less than $n$, the traditional choice of $p = 1/n$ will be far too
conservative, and the expected number of function evaluations until
the optimum is found will be $\Theta(nk)$. On the other hand, choosing
a higher value of $p$ such that $p=\omega(1/n)$ will not work if the
adversary chooses a value of $k$ quite close to $n$, since in this
case the EA will flip the leading 1-bits with too high probability and
have exponential runtime. However, several extensions of the $(1+1)$
EA have been proposed which are more effective for optimisation in
this uncertain environment. In \cite{bib:Doerr2015}, Doerr et
al. consider two different variants of the $(1+1)$ EA, one which
assigns different flip probabilities to each bit, and one which
samples a new bitflip probability from a distribution $Q$ in each
generation, both of which they show to have an expected runtime of
$O(k^2(\log k)^{1+\ep})$ on $\textsc{LeadingOnes}_k$. They also show
the $\log^{1+\ep}k$ term can be further reduced by more carefully
choosing the positional bitflip probabilities or the distribution $Q$;
however, in a follow-up work, it is shown that the upper bound for
both of these algorithms is nearly tight, that is, the expected
runtime is $\omega(k^2\log k)$ \cite{bib:Doerr2017}.  In
\cite{bib:Einarsson2018}, a different sort of self-adjusting $(1+1)$
EA is introduced for the hidden subset problem on the class of linear
functions. Rather than adjusting the mutation rate
in each generation during the actual search process, the algorithm
instead spends $O(k)$ generations approximating the hidden value $k$,
and then $O(k\log k)$ generations actually optimising $f_k$ now that
$k$ is approximately known. This algorithm not only improves the bound
from $O(k(\log k)^{2+\ep})$ in \cite{bib:Doerr2015} to $\Theta(k\log k)$
for $\textsc{OneMax}_k$ under the hidden subset model, but the
implicit constants of $(1\pm o(1))en\ln n$ are found as well, matching
the performance of a $(1+1)$ EA which knows $k$ in advance. However,
it remained to be demonstrated whether an EA could similarly solve the
$\textsc{LeadingOnes}_k$ problem at no extra cost when $k$ was unknown.

\subsection{Structure of the Paper}
\label{sec:structure}

Section \ref{sec:prelim} introduces notation, a formal description of
the self-adaptive algorithm (Algorithm~\ref{algo:stepwise-sa}), and
the analytical tools we used. Section~\ref{sec:main-result} provides
our main result, that Algorithm~\ref{algo:stepwise-sa} optimises
$\textsc{LeadingOnes}_k$ against an adversary in expected time
$O(k^2)$. Section~\ref{sec:experiments} is an experimental study on
theoretical benchmark functions, first illustrating the evolution of
the mutation rate throughout the optimisation process, then comparing
the average runtime during optimisation against an adversary of the
algorithm to some classic EAs and to the adaptive $(1+1)_{\alpha}$ EA
\cite{bib:Doerr2018}. Section~\ref{sec:conclusion}
concludes the paper.

\section{Preliminaries}\label{sec:prelim}

\subsection{General Notation}
\label{sec:notation}

For any $n\in\mathbb{N}$, let $[n]:=\{1,...,n\}$ and
$[0..n]:=\{0\}\cup [n]$.  The natural logarithm is denoted by
$\ln(.)$, and the logarithm base 2 by $\log(.)$. The Iverson bracket
is denoted by $[.]$, which is equal to 1 if the statement in the
brackets is true, and 0 otherwise.
The search space throughout this work is $\mathcal{X}:=\{0,1\}^n$, and we
refer to $x=(x_1,...,x_n)$ in $\mathcal{X}$ as a bitstring of length
$n$. Since we are interested in searching the space of mutation rates
along with the set of bitstrings, it will be convenient to define
an extended search space of
\begin{align}
  \mathcal{Y} := \sspace\times [\ep, 1/2].
\end{align}
The parameter $\ep = c/n$, where $c<1$ is a small constant with
respect to $n$, is necessary only for technical reasons in our
analysis. The Hamming distance between two bitstrings $x, x'$ is
denoted by $H(x,x')$. All asymptotic notation throughout this work is
with respect to $n$, the size of the problem space.  The
\emph{runtime} of a search process is defined as the number of fitness
evaluations before an optimal search point is found, denoted by
$T$. Generally we are concerned with the \emph{expected runtime},
$\expect{T}$.

\subsection{A Self-adaptive $(\mu, \lambda)$  EA}
\label{sec:alg-prelim}

We consider a non-elitist EA using self-adaptation of
mutation rates, outlined in Algorithm \ref{algo:stepwise-sa}. We refer
to a \emph{population} as a vector $P\in\mathcal{Y}^\lambda$, where
$\lambda\in\N$ is the \emph{population size}, and where the $i$-th
element $P(i)$ is called the $i$-th \emph{individual}. For an
individual $(x, \chi/n)\in\extsspace$,
we refer to
$\chi / n \in[\ep, 1/2]$ as the \emph{mutation rate}, and $\chi$
as the \emph{mutation parameter}. In each
\emph{generation} $t\in\N_0$, Algorithm \ref{algo:stepwise-sa} creates
the next population $P_{t+1}$ by independently creating $\lambda$ new
individuals according to a sequence of operations \emph{selection},
\emph{adaptation}, and \emph{mutation}.

\begin{algorithm}
  \caption{$(\mu, \lambda)$ Self-adaptive EA}
  \begin{algorithmic}[1]
    \Require Fitness function $f:\mathcal{X}\rightarrow\R$.
    \Require Population sizes
    $\mu,\lambda\in\mathbb{N}$, where $1\leq\mu\leq \lambda$.
    \Require Adaptation parameters $A > 1$, and
    $b,\pinc\in(0,1)$.
    \Require Initial population $P_0\in \mathcal{Y}^\lambda$.
    \For{$t$ in $0,1,2,\dots$
      until termination condition met}
    \State Sort $P_t$ st. $P_t(1)\succeq 
    \cdots \succeq P_t(\lambda),$ according to
    (\ref{eq:order}).
    \label{algo:order}
    \For{$i$ in $1,\dots,\lambda$}
    \State Set $(x, \chi/n) := P_t(I_t(i))$, $I_t(i)\sim\text{Unif}([\mu])$. \label{algo:select}
    \State Set $\chi' :=
    \begin{cases}
      \min\{A\chi, n/2\} \text{ with probability } \pinc\\
      \max\{b\chi, \ep n\} \text{ otherwise.}
    \end{cases}$
    \label{algo:adapt}
    \State Create $x'$ by independently flipping each bit of $x$ with
    probability $\chi'/n$.
    \label{algo:mutate}
    \State Set $P_{t+1}(i):=(x',\chi'/n)$.
    \label{algo:update}
    \EndFor
    \EndFor
  \end{algorithmic}
  \label{algo:stepwise-sa}
\end{algorithm}

\subsubsection{Selection}

We consider a variant of the standard $(\mu, \lambda)$ selection
scheme, where the $\mu\leq\lambda$ best individuals are chosen
according to fitness, with ties broken according to the individual
with higher mutation rate. More precisely, the population is first
sorted according to the ordering
\begin{gather}
  (x,\chi) \succeq (x',\chi') 
  \Leftrightarrow \label{eq:order}
  f(x) > f(x') \vee (f(x) = f(x') \wedge \chi \geq \chi'), 
\end{gather}
where ties of $f(x) = f(x')$ and $\chi = \chi'$ are broken
arbitrarily. Then, each parent is chosen uniformly from the $\mu$ top
individuals $P_t(1),\ldots,P_t(\mu)$.

We quantify the
\emph{selective pressure} of the selection mechanism 
using the \emph{reproductive rate}.
\begin{definition}[\cite{bib:Lehre2010}]
  \label{def:reproductive-rate}
  The \emph{reproductive rate} of Algorithm \ref{algo:stepwise-sa} is
  $
    \alpha_0 := \max_{1\leq i \leq \lambda}\expect{R_t(i)},
  $
  where
  $R_t(i) := \sum_{j=1}^{\lambda}[I_t(j) = i]$.
\end{definition}
That is, $\alpha_0$ is the expected number of times per generation an
individual with the highest selection probability is chosen in step
\ref{algo:select} of Algorithm \ref{algo:stepwise-sa}.  A well-known
fact is that the reproductive rate of the standard ($\mu,\lambda$) EA
is $\lambda/\mu$ (Lemma 7 in \cite{bib:Lehre2011}). This is also
the case for Algorithm \ref{algo:stepwise-sa}.

\subsubsection{Adaptation}

Each chromosome $(x,\chi/n)$ carries both a search point $x$ and a
mutation parameter $\chi$. In order for the population to explore different
mutation rates, it must be possible for the offspring to inherit a
``mutated'' mutation parameter $\chi'$ different from its parent.
For the purpose of the theoretical analysis, we are looking for the
simplest possible update mechanism, which is still capable of adapting
the mutation rates in the population.

We will prove that the following simple multiplicative update scheme
suffices: given a parent with mutation parameter $\chi$, the offspring inherits an
increased mutation parameter $A\chi$ with probability $\pinc$, and a
reduced mutation parameter $b\chi$ with probability $1-\pinc$, where
$A$ and $b$ are two parameters satisfying $0<b<1<A$.
We choose the parameter names $A$ and $b$ for consistency with
the adaptive $(1+1)_{\alpha}$ EA already introduced in
\cite{bib:Doerr2018}, which similarly changes the mutation rate in
this step-wise fashion. However, unlike
Algorithm~\ref{algo:stepwise-sa}, the $(1+1)_{\alpha}$ EA changes
the mutation rate based on whether the
offspring is fitter than the parent.

Our goal is that the evolutionary algorithm adapts the mutation
parameter $\chi$ to the problem at hand, so that it is no longer
necessary to set the parameter $\chi$ manually.  It may seem
counter-productive to replace one mutation parameter by introducing
three new adaptation mechanism parameters $A,b,$ and $\pinc$. However,
we will show that while the mutation parameter $\chi$ must be tuned
for each problem, the same fixed setting of the parameters $A,b,\pinc$
is effective across many problems.  We conjecture that the
self-adaptive EA will be effective with other adaptation
mechanisms. For example, rather than multiplying by the constants $A$
and $b$, we could multiply the mutation parameter by a factor with
$\log$-normal distribution as was originally done in
\cite{Back1996}. We suspect that many adjustment mechanisms which
favour taking small steps from the current mutation rate could be
analysed similarly to the analysis presented in this work.

\subsubsection{Mutation}

The mutation step is when a new candidate solution is actually created
by our algorithm. We consider standard bit-wise mutation, where
a parent $x\in\{0,1\}^n$ with mutation rate $\chi/n$ produces an offspring
$x'\in\{0,1\}^n$ by flipping each bit of $x$ independently 
with probability $\chi/n$. We adopt the notation from \cite{bib:Lehre2010},
and consider the offspring $x'$ a random variable
$x'\sim\pmut(x,\chi)$, with distribution
\begin{align*}
 \prob(x'=\pmut(x,\chi)) := \left(\frac{\chi}{n}\right)^{H(x,x')}\left(1-\frac{\chi}{n}\right)^{n-H(x,x')}.
\end{align*}

\subsection{Level-based Analysis}
\label{sec:level-based}

We analyse the runtime of Algorithm~\ref{algo:stepwise-sa} using 
\emph{level-based analysis}. Introduced by Corus et al.
\cite{bib:Corus2017}, the level-based theorem is a general tool for
deriving upper bounds on the expected runtime for non-elitist
population-based evolutionary algorithms,
and has been applied to a wide range of algorithms, including to GAs
\cite{bib:Corus2017}, and EDAs \cite{bib:Dang2018}.

The theorem can be applied to any population-based stochastic process
$(P_t)_{t\in\N}$, where individuals in $P_{t+1}$ are sampled
independently from a distribution $D(P_t)$, where $D$ maps populations
to distributions over the search space. In the case of our algorithm,
$D$ is the composition of selection, adaptation, and mutation. The
theorem also assumes a partition $(A_1,\dots,A_m)$ of the finite
search space into $m$ subsets, also called \emph{levels}. Usually,
this partition is over the function domain $\sspace$, but since we are
concerned with tracking the evolution of the population over the
2-dimensional space of bitstrings and mutation rates, we will
rather work with subsets of $\extsspace$.

Given any subset $A \subseteq\extsspace$, we slightly abuse
notation and let
$
  |P \cap A|:= |\{i\in[\lambda] \mid P(i) \in A\}| 
$
denote the number of
individuals in a population $P\in\mathcal{Y}^\lambda$ that belong to the subset $A$. Given
a partition of the search space $\mathcal{X}$ into levels
$(A_1, \ldots, A_m)$, we define for notational convenience
$A_{\geq j}:=\bigcup_{i=j}^mA_i$ and $A_{> j}:=\bigcup_{i=j+1}^mA_i$.

\begin{theorem}[\cite{bib:Corus2017}]
  \label{thm:level-based}
  Given a partition
  $(A_1,\dots,A_{m})$ of $\mathcal{X}$, define $T :=
  \min\{t\lambda \mid P_t\cap A_{m}\neq\emptyset\}$, where for all
  $t\in\mathbb{N}$, $P_t\in\mathcal{X}^\lambda$ is the
  population of Algorithm 1 from \cite{bib:Corus2017} in generation $t$.
  If there
  exist $z_1,\dots,z_{m-1},\delta \in(0,1]$,
  and $\gamma_0 \in (0,1)$
  such that
  for any population $P\in \mathcal{X}^\lambda$,
  \begin{description}
  \item[(G1):]
    for each $j \in [m-1]$,
    if
    $|P\cap A_{\geq j}  | \geq \gamma_0\lambda$,
    then
    \[
      \displaystyle \Pr_{y\sim D(P)}\left( y\in A_{\ge j+1}\right)
      \geq z_j,
    \]
  \item[(G2):]
    for each $j \in [m-2]$, and all $\gamma\in(0,\gamma_0]$\\
    if
    $|P\cap A_{\ge j}  |  \geq \gamma_0\lambda$ and
    $|P\cap A_{\ge j+1}|  \geq   \gamma\lambda$, then
    \[
      \Pr_{y\sim D(P)}\left( y\in A_{\geq j+1}\right)
      \geq (1+\delta)\gamma,
    \]
  \item[(G3):] the population size $\lambda\in\mathbb{N}$ satisfies
    \[
      \lambda \geq \left(\frac{4}{\gamma_0\delta^2}\right)
      \ln\left(\frac{128m}{z_*\delta^2}\right), \text{ where }
      z_*:=\min_{j\in[m-1]} \{z_j\},
    \]
  \item[\textit{then}]
    $
      \expect{T} \leq \left(\frac{8}{\delta^{2}}\right) \sum_{j=1}^{m-1}
      \left(\lambda\ln\left(\frac{6\delta\lambda}{4+z_j\delta\lambda
          }\right)+\frac{1}{z_j}\right).
    $
  \end{description}
\end{theorem}

\section{Runtime Analysis on \leadingones}
\label{sec:main-result}

We now introduce our main result, which is an upper bound on the
optimisation time of Algorithm \ref{algo:stepwise-sa} on the
$\textsc{LeadingOnes}_k$ problem. Note that for population size
$\lambda = c\ln(n)$ and problem parameter $k \geq \log(n)\log(n)$,
the bound in the theorem simplifies to $O(k^2)$ which is
asymptotically optimal among all unary unbiased black-box algorithms,
regardless of whether the parameter $k$ is known
\cite{lehre_black-box_2012}.

\begin{theorem}
  \label{thm:k-lo}
  Algorithm \ref{algo:stepwise-sa} with
  $\frac{\lambda}{\mu} = \alpha_0 \geq 4$, constant parameters
  $A, b, \pinc\in\mathbb{R}$ satisfying $A > 1$,
  $(1+\delta)/ \alpha_0 < \pinc < 2/5$, and
  $0 < b < 1 / (1 + \sqrt{1 / (\alpha_0(1 - \pinc))})$ for some
  $\delta\in(0,1)$, parent population size $\mu = \Omega(\log(n))$, and $\lambda \geq c\ln(n)$
  for a large enough constant $c$, for any $k\in\N$, has expected
  runtime $O(k\lambda\log(n\lambda) + k^2)$ on
  $\textsc{LeadingOnes}_k$.
\end{theorem}

The proof of Theorem \ref{thm:k-lo} is structured as follows. In
Section \ref{sec:def-sublevels}, in order to apply the level-based
analysis and track the population's progress over a two-dimensional
landscape, we begin by defining a partition of $\extsspace$. Since our
partition is more involved than those usually applied to the
level-based theorem, we also verify it is truly a partition. In
Section \ref{sec:control-bad}, we identify a region of the search
space where individuals have a mutation rate which is too high with
respect to their fitness, then show that with overwhelming
probability, individuals in this region will not dominate the
population. In Section \ref{sec:calc-probs}, the main technical
section, we calculate the probabilities of a parent individual
producing an offspring in a level at least as good as its own, and of
producing an offspring in a strictly better level.
Finally in Section~\ref{sec:apply-level-based} we put
everything together and apply Theorem~\ref{thm:level-based} to our
partition to obtain an upper bound on the expected runtime.

\subsection{Partitioning the search space into levels}
\label{sec:def-sublevels}

We now partition the two-dimensional search space
$\extsspace=\mathcal{X}\times[\ep,1/2]$ into ``levels'', which is
required to apply
Theorem~\ref{thm:level-based}.
The proof of Theorem~\ref{thm:level-based} uses the levels to
measure the progress of the population through the search space. The
progress of Algorithm \ref{algo:stepwise-sa} depends both on the
fitness of its individuals, as well as on their mutation rates. We
start by defining a partition on the search space $\mathcal{X}$,
into $k+1$ canonical fitness levels, for $j\in[0..k]$,
\begin{align*}
  A_j := \{x\in\mathcal{X} \mid \leadingones(x)=j\}.
\end{align*}
These fitness levels will be used later to define a partition on the extended
search space $\mathcal{Y}$.

The probability of a ``fitness upgrade'', i.e., that a parent produces
an offspring which is fitter than itself depends on the mutation rate
of the parent. If the fitness of the parent is $j$, but its mutation
rate is significantly lower than $1/j$, then the algorithm will lose
too much time waiting for a fitness upgrade, and should rather produce
offspring with increased mutation rates. Conversely, if the mutation
rate is significantly higher than $1/j$, then the mutation operator is
too likely to destroy the valuable bits of the parent.

To make this intuition precise, we will define for each fitness level
$j\in[0..k-1]$, two \textit{threshold} values $\theta_1(j)$ and
$\theta_2(j)$. These values will be defined
such that when the mutation rate satisfies $\chi/n\in[\ep,\theta_1(j))$,
then the mutation rate is too low for a speedy fitness upgrade, \todo{}
when the mutation rate satisfies $\chi/n\in[\theta_1(j),\theta_2(j)]$,
then the mutation rate is ideal for a fitness upgrade, and when
the mutation rate satisfies $\chi/n\in(\theta_2(j),1/2]$,
then the mutation rate is too high. To not distract from our
introduction of the levels, we postpone the detailed derivation of the expressions
of the threshold values to Section \ref{sec:calc-probs}, and simply
assert that they satisfy the following conditions for all $j\in[0..k-1]$,
\begin{enumerate}
\item $\ep < \theta_1(j)< \min(1/2,\theta_2(j))$
\item $\theta_1(j)> \theta_1(j+1)$ 
\item $\theta_2(j)> \theta_2(j+1)$,
\end{enumerate}

Condition (1) states that $[\theta_1(j),\theta_2(j)]$ forms an
interval which always overlaps with the range of mutation rates
reachable by Algorithm~\ref{algo:stepwise-sa}, while (2) and (3)
state that both $\theta_1$ and $\theta_2$ are monotonically
decreasing functions. 

To reflect the progress of the population in terms of increasing the
mutation rate towards the ``ideal'' interval
$[\theta_1(j),\theta_2(j)]$ within a fitness level $j$, we partition
the extended search space $\mathcal{Y}$
into sub-levels $A_{(j,\ell)}$.
The lowest sub-level $A_{(j,0)}$ corresponds to
individuals with fitness $j$ and mutation rate in the interval from $\ep$
to $A\ep$. If the mutation rate is increased by a factor of $A$ from
this level, one reaches the next sub-level $A_{(j,1)}$. In general,
sub-level $A_{(j,\ell)}$ corresponds to individuals with fitness $j$
and mutation rates in the interval from $A^{\ell-1}\ep$ to
$A^{\ell}\ep$, etc. After the mutation rate has been increased a
certain number of times, which we call the ``depth'' of fitness
level $j$, one reaches the ideal interval $[\theta_1(j),\theta_2(j)]$. 

\begin{definition}
  For each $j\in[k-1]$,
  the \emph{depth} of level $j$ is the unique positive integer
  \begin{equation}
    \label{eq:depth-def}
    d_j := \min\left\{\ell\in\N\mid \ep A^{\ell}\geq \theta_1(j)\right\},
  \end{equation}
  where $A$ is the step-size parameter from Algorithm
  \ref{algo:stepwise-sa}.
\end{definition}

Our next step in introducing the levels which build our partition of
$\extsspace$ is to distinguish between two conceptual types of levels,
namely, between \textit{low} levels and \textit{edge} levels. The low
levels represent regions of $\extsspace$ where individuals have
mutation rate below $\theta_1(j)$, i.e.,
can still raise mutation rate while maintaining fitness with good
probability. For each fitness value $j\in[0..k-1]$, there are $d_j-1$
low levels.

Edge levels form a region of search points where the mutation rate is
neither too low nor too high with respect to $j$, i.e., in the ideal
interval from $\theta_1(j)$ to $\theta_2(j)$. It is these search points which are
best equipped for upgrading from fitness level $j$ to a strictly
better level. At the same time, increasing mutation further would put
these individuals in danger of ruining their fitness. To progress from
an edge level, an individual must strictly increase its fitness.

The final technicality to discuss before defining our partition is
where to place an individual with fitness $j$ and mutation rate
$\chi/n>\theta_2(j)$.
We avoid placing such individuals into any of the low or edge levels
corresponding to fitness level $j$. However, due to the
conditions (1) through (3) we imposed on $\theta_2$, there exists
some lower fitness value $j' < j$ such that
$\theta_1(j')\leq \chi/n \leq \theta_2(j')$. This means that $j'$ is
the largest number of bits the individual
will be able to
maintain with good enough probability. We will therefore add such
an individual
to a level corresponding to fitness level $j'$.

We now define the ``low levels'' and the
``edge levels'' on the extended search space $\mathcal{Y}=\mathcal{X}\times[\ep,1/2]$.
\begin{definition}
\label{def:sublevels}
For $j\in[0..k-1]$ and $\ell\in[d_j-1]$, we define the low levels as
\begin{gather}
  A_{(j, \ell)} := A_j \times \left[A^{\ell-1}\ep, \min \left(A^{\ell}\ep,\theta_1(j)\right)\right)
  \label{eq:L-def}
\end{gather}
and for $j\in[0..k-1]$, we define the edge levels as
\begin{align}
  \label{eq:edge-level-def}
  A_{(j,d_j)} := & A_j   \times \left[\theta_1(j), \min\left(\frac{1}{2},\theta_2(j)\right)\right] \cup\\
                & A_{>j}\times \left(\min \left(\frac{1}{2},\theta_2(j+1)\right), \min\left(\frac{1}{2},\theta_2(j)\right)\right] \nonumber
\end{align}
\end{definition}

Additionally, since all search points with $\textsc{LO}_k(x) \geq k$ are
globally optimal,
we simply define a new set $A_{(k, 1)}\in\extsspace$ where
\begin{equation}
  \label{eq:opt-def}
  A_{(k, 1)} := \{(x, \chi/n) \mid \textsc{LO}_k(x) = k\}.
\end{equation}

Thus, we define our partition of $\extsspace$ to consist of all sets
$A_{(j, \ell)}$ from Definition \ref{def:sublevels}, where
$j\in[0..k-1]$ and $\ell\in[d_j]$, and $A_{(k,1)}$. Hence,
there are $m:=\left(\sum_{j=0}^{k-1}d_j\right) + 1$
levels.

We now prove that the levels form a partition of the
extended search space $\mathcal{Y}$. To simplify the proof, we first
define some upper bounds for the sub-levels. For all $\ell\leq d_j$,
the definition of $d_j$ and Lemma \ref{lemma:r-bounds}
(\ref{item:r3b}) imply that
\begin{align*}
  A^{\ell-1}\ep< \theta_1(j)<\theta_2(j+1),
\end{align*}
which leads to the upper bounds
\begin{align}
  A_{(j,\ell)}
  &\subset A_{\geq j}\times \left[A^{\ell-1}\ep,\min\left(\frac{1}{2},\theta_2(j)\right)\right] \label{eq:level-set-bound2a}\\
  &\subseteq A_{\geq j}\times \left[A^{\ell-1}\ep,\theta_2(j)\right].\label{eq:level-set-bound2}
\end{align}
Furthermore, for all $v<d_j$, Eq. \ref{eq:L-def} gives the trivial upper bound \todo{}
\begin{align}
  A_{(j,v)} & \subset A_j\times (-\infty,A^{u}\ep). \label{eq:level-set-bound1}
\end{align}
Lemma \ref{lemma:non-intersecting-levels}
and Lemma \ref{lemma:levels-cover-searchspace} imply that we have a
partition of the search space.

\begin{lemma}\label{lemma:non-intersecting-levels}
  \todo{}
  For all $(u,v)\neq (j,\ell)$, it holds $A_{(u,v)}\cap A_{(j,\ell)}=\emptyset$.
\end{lemma}
\begin{proof}
  The following three cases use that $(X\times Y)\cap (U\times V)=(X\cap U)\times (Y\cap V)$.
  \newline
  \underline{Case 1:} $u=j$. Assuming w.l.o.g. that $v<\ell\leq d_u$,
  the bounds (\ref{eq:level-set-bound2}) and  (\ref{eq:level-set-bound1}) give
  \begin{align*}
    A_{(u,v)}\cap A_{(j,\ell)}
    & \subset A_u\times \left((-\infty,A^v\ep)\cap [A^{\ell-1} \ep,\infty)\right)
     = \emptyset,
  \end{align*}
  where the last equality follows from $A^v\ep\leq A^{\ell-1}\ep$.
 
\underline{Case 2:} $u+1\leq j$ and $v<d_u$.
In this case, the bounds (\ref{eq:level-set-bound2}) and
(\ref{eq:level-set-bound1}) give
  \begin{align*}
    A_{(u,v)}\cap A_{(j,\ell)}
    & \subset (A_u\cap A_{\geq j})\times (-\infty,\infty)
     = \emptyset,
  \end{align*}  
  where the last equality follows from $A_u\cap A_{\geq j}=\emptyset.$

\underline{Case 3:} $u+1\leq j$ and $v=d_u$. In this
case, the bound (\ref{eq:level-set-bound2a}) and the definition of
$A_{(u,d_u)}$ in
(\ref{eq:edge-level-def}) give
  \begin{align*}
    A_{(u,v)}\cap A_{(j,\ell)}
    & \subset A_{(u,d_u)} \cap \left(A_{\geq j}\times \left(-\infty,\min\left(\frac{1}{2},\theta_2(j)\right)\right]\right)\\
    & \subset \left(A_{>u}\cap A_{\geq j}\right)\times \\
    &\quad\bigg(\left(\min\left(\frac{1}{2},\theta_2(u+1)\right),\infty\right),\\
    &\quad\quad\left(-\infty,\min\left(\frac{1}{2},\theta_2(j)\right)\right]\bigg)
     = \emptyset,
  \end{align*}
  where the last equality follows from the fact that the function $\theta_2$ decreases
  monotonically, thus $\theta_2(j)\leq \theta_2(u+1)$.
\end{proof}

\begin{lemma}\label{lemma:levels-cover-searchspace}
  $\bigcup_{j=0}^k\bigcup_{\ell=1}^{d_j} A_{(j,\ell)}=\mathcal{Y}$.
\end{lemma}
\begin{proof}
  We prove by induction on $u\leq k$ that
  \begin{align}
    \bigcup_{j=u}^k\bigcup_{\ell=1}^{d_j} A_{(j,\ell)}=A_{\geq u}\times \left[\ep,\min\left(\frac{1}{2},\theta_2(u)\right)\right].\label{eq:union-induction}
  \end{align}
  For the base step, \eqref{eq:union-induction} holds when $u=k$ because by the level definition \todo{}
  \begin{align}
    \bigcup_{\ell=1}^{d_k} A_{(k,\ell)} & = A_k\times \left[\ep,\min\left(\frac{1}{2},\theta_2(k)\right)\right].
  \end{align}
  For the inductive step, assume that (\ref{eq:union-induction}) holds
  for some $u\in[k]$. We then have
  \begin{align*}
    &\bigcup_{j=u-1}^k\bigcup_{\ell=1}^{d_j} A_{(j,\ell)}
      = \left(\bigcup_{\ell=1}^{d_j} A_{(u-1,\ell)}\right)\\
    &\quad\quad\quad\quad\quad\quad\quad~ \cup\left(A_{\geq u}\times  \left[\ep,\min\left(\frac{1}{2},\theta_2(u)\right)\right]\right)\\
    &\quad = \left(A_{u-1}\times \left[\ep,\min\left(\frac{1}{2},\theta_2(u-1)\right)\right]\right) \\
    &\quad\quad \cup \left(A_{\geq u}\times \left(\min\left(\frac{1}{2},\theta_2(u)\right),\min\left(\frac{1}{2},\theta_2(u-1)\right)\right]\right)\\
    &\quad\quad \cup\left(A_{\geq u}\times  \left[\ep,\min\left(\frac{1}{2},\theta_2(u)\right)\right]\right)\\
    &\quad = A_{\geq u-1}\times \left[\ep,\min\left(\frac{1}{2},\theta_2(u-1)\right)\right].
  \end{align*}
  By induction, (\ref{eq:union-induction}) holds for $u=0$. The proof
  is now complete by noting that $\theta_2(0)>1/2$ due to Lemma~\ref{lemma:r-bounds}~(\ref{item:r0}).
\end{proof}

The level-based theorem assumes that the levels are totally ordered,
however we have introduced two-dimensional levels. We
will order the levels using the lexicographic order $\succeq$ defined for
$j,j',\ell,\ell'$ by
\begin{align*}
  A_{(j',\ell')} \succeq A_{(j,\ell)} \;\Longleftrightarrow\;
  (j' > j) \vee (j' = j \wedge \ell' \geq \ell).
\end{align*}
Also, it will be convenient to introduce the notation
\begin{equation}
  \label{eq:Aorbetter-def}
  A_{\geq (j,\ell)}:= \bigcup \left\{ A_{(j', \ell')}\mid A_{(j', \ell')}\succeq A_{(j, \ell)} \right\}.
\end{equation}

\subsection{Survival and Upgrade Probabilities}
\label{sec:calc-probs}

Having partitioned the search space into levels, the next steps in
applying the level-based theorem are to prove that conditions (G1) and
(G2) are satisfied. This amounts to estimating the probability
that an offspring does not decrease to a lower level
(condition (G2)), and the probability that it upgrades to a
strictly better level (condition (G1)). It will be
convenient to introduce a measure for the probability of reproducing a
bitstring of equal or better fitness by applying a mutation rate $\chi/n$.
\begin{definition}
  \label{eq:r-def}
  \todo{} For all $j\in[0..n]$ and $\chi\in[\ep n,n/2]$, we define
  the
  \emph{survival probability} as
  \begin{align*}
    r(j, \chi) := \min_{x\in A_j}\prob_{x'\sim\pmut(x, \chi)}(x'\in A_{\geq j}).
  \end{align*}
\end{definition}
For \leadingones, it is straightforward to show that
$r(j, \chi) = \left(1 - \chi / n\right)^{j}$.

\begin{figure}
  \centering
  \includegraphics[width=8cm]{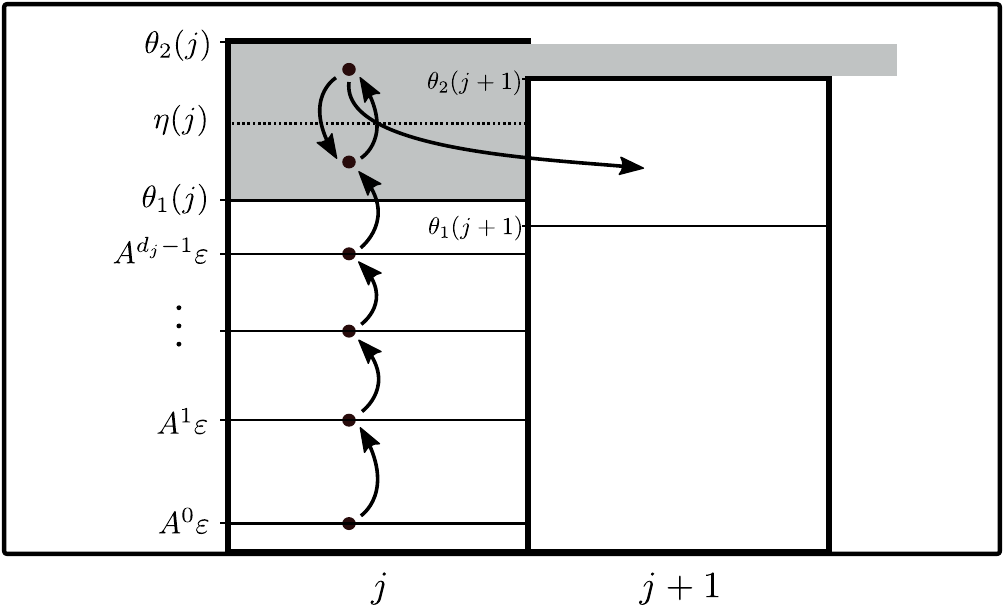}
  \caption{A typical lineage of individuals. The population makes progress
    keeping its $j$ leading 1-bits and raising mutation
    rate. Once in the edge region (grey), it remains in the edge region
    while increasing or decreasing mutation. Individuals in the edge
    region can move to a strictly better level by obtaining at least
    $j+1$ leading 1-bits and lowering mutation rate.}
  \label{fig:levels2}
\end{figure}

Fig. \ref{fig:levels2} illustrates a typical lineage of individuals,
from fitness level $j$ to fitness level $j+1$. Starting from some low
mutation rate in fitness level $j$, the mutation rate is increased by
a factor of $A$ in each generation, until the mutation rate reaches
the interval $[\theta_1(j),\theta_2(j)]$, i.e., the edge level. The
lineage circulates within the edge level for some generations,
crossing an intermediary value $\eta(j)$, before the fitness improves,
and fitness level $j+1$ is reached.

It is critical to show that with sufficiently high probability, the
lineage remains in the edge level before upgrading to fitness level
$j+1$.  This is ensured by the bounds in
Lemma~\ref{lemma:r-bounds}. Statement (\ref{item:r3b}) implies that we
cannot overshoot the edge level by increasing the mutation
rate. Statement (\ref{item:r1}) implies that below the intermediary
mutation rate $\eta(j)$, the mutation rate can still be increased by a
factor of $A$.  Conversely, statement (\ref{item:r2}) means that above
the intermediary mutation rate $\eta(j)$, it is safe to decrease the
mutation rate by a factor of $b$. Statements (\ref{item:r4}) and
(\ref{item:r5}) ensure that
an individual in an edge level can always either increase or decrease
mutation rate for there to be a sufficiently high probability of
maintaining the individual's fitness value. Finally, statements
(\ref{item:r0b}) and (\ref{item:r0c}) imply that within the edge
level, the upgrade probability is $\Theta(1/j)$.

Before we can prove these statements, recall that we have delayed
formally defining the functions $\theta_1$, $\eta$, or $\theta_2$. In
order to derive the claimed bounds for Lemma \ref{lemma:r-bounds}, we
do this now. For $j\geq 1$, let
  \begin{align}
    \eta(j) &:= \frac{1}{2A}\left(1 - \left(\frac{1+\delta}{\alpha_0\pinc}\right)^{1/j}\right)\label{eq:eta-def}\\
    \theta_1(j) &:= b \eta(j)\label{eq:theta1-def}\\
    \theta_2(j) &:=  1 - q^{1/j}\label{eq:theta2-def}
  \end{align}
  where
  \begin{align*}
    q  :=  \frac{1- \zeta}{\alpha_0},\quad    
    r_0  := \frac{1+\delta}{\alpha_0(1 - \pinc)}, \text{ and }
  \end{align*}
  \begin{align}
    \zeta  := 1 - \alpha_0(r_0)^{1+\sqrt{r_0}}\label{eq:zeta-def}.
  \end{align}
  Furthermore, for the special case $j=0$, define
  \begin{align*}
    \eta(0)      := \frac{\eta(1)}{A},\quad
    \theta_1(0)  := b\eta(0),\quad\text{and}\quad
    \theta_2(0)  := \frac{\theta_2(1)}{b}.
  \end{align*}
  Note that these definitions, along with statement
  \eqref{item:r0} of Lemma \ref{lemma:r-bounds}, ensure $\theta_1$
  and $\theta_2$ satisfy the informal conditions from Section
  \ref{sec:def-sublevels} for $\ep$ small enough.

\begin{lemma}
  \label{lemma:r-bounds}
  Let $A>1$, $b<1$, and $\pinc\in(0,1)$ be constants satisfying the
  constraints in Theorem \ref{thm:k-lo}. Then there exists a constant
  $\delta\in(0,1/10)$ such that for all $j\in[0..k-1]$ and
  $\chi/n\in[\ep, 1/2]$,
  \begin{enumerate}[(i)]
  \item\label{item:r0} $\theta_1(0)<\eta(0) < 1/2 < \theta_2(0)$,
  \item\label{item:r0b} $\theta_2(j)=\Omega(1/j)$
  \item\label{item:r0c} $\theta_1(j)=O(1/j)$
  \item\label{item:r1} $A\eta(j) \leq \theta_2(j)$,
  \item\label{item:r2} $b\eta(j) \geq \theta_1(j)$,
  \item\label{item:r3} $b\theta_2(j) < \theta_2(j+1)$,
  \item\label{item:r3b} $A\theta_1(j)\leq \theta_2(j+1)$,
  \item\label{item:r4} if $\frac{\chi}{n} \leq \eta(j),$ then $r(j, A\chi) \geq \frac{1+\delta}{\alpha_0\pinc}$, and
  \item\label{item:r5} if $\frac{\chi}{n} \leq \theta_2(j)$, then $r(j, b\chi) \geq \frac{1 + \delta}{\alpha_0(1 - \pinc)}$.
  \end{enumerate}
\end{lemma}

\begin{proof}
  Before proving statements (\ref{item:r0})--(\ref{item:r5}),
  we derive bounds on the three constants $q, \zeta,$ and $r_0$.
  By the assumptions $\pinc < 2/5$ and $\alpha_0\geq
  4$ from Theorem \ref{thm:k-lo} and $\delta< 1/10$, 
  \begin{align}
    r_0 < \frac{11}{6\alpha_0} < 1.
  \end{align}
  Furthermore, since $r_0<1$ and $\alpha_0\geq 4$, we have 
  \begin{align}
    \zeta & > 1 - \alpha_0(r_0)^2    
           > 1 - \frac{1}{\alpha_0}\left(\frac{11}{6}\right)^2 
            \geq \frac{23}{144}.\label{eq:zeta-lower-bound}
  \end{align}
  Finally, since $\delta,\zeta,\pinc\in(0,1)$, we have from the
  definitions of $r_0$ and $q$ that
  \begin{align}
    0 < q < r_0.\label{eq:q-r-rel}
  \end{align}
  From the definition of the functions $\theta_1$, $\eta$, and the
  constant $\delta\in(0,1/10)$, it follows that
  \begin{align*}
    \theta_1(0) & < \eta(0)
      < \eta(1)
     < \frac{1}{2A}\left(1 - \frac{1}{\alpha_0\pinc}\right)
     < \frac{1}{2}.
  \end{align*}
  Also, we have from the definition of $q$, the constraint
  $\alpha_0\geq 4$ from Theorem~\ref{thm:k-lo}, and the bound
  $\zeta>23/144$ from (\ref{eq:zeta-lower-bound}) that
  \begin{align*}
    \theta_2(0)
    & > \theta_2(1) 
     > 1-q
     = 1-\frac{1-\zeta}{\alpha_0}
     > 1-\frac{1-\frac{23}{144}}{4}
      = \frac{455}{576}.
  \end{align*}
  Thus, we have proven statement (\ref{item:r0}).

  Statement (\ref{item:r0b}) follows directly from
  Lemma~\ref{lemma:inv-bound}, the definition of $\theta_2$ and the
  constant $q$,
  \begin{align*}
    \theta_2(j) & = 1-q^{1/j}\geq \ln(1/q)/j = \Omega(1/j).
  \end{align*}

  For statement (\ref{item:r0c}), we define
  $c:=\frac{1+\delta}{\alpha_0\pinc}<1$, and observe that the inequality $e^{x}\geq 1+x$ implies 
  \begin{align*}
    \theta_1(j)< 1-c^{1/j} = 1-e^{(1/j)\ln(c)} \leq -(1/j)\ln(c)=O(1/j).
  \end{align*}
  
  For statement (\ref{item:r1}), first note that Eq. (\ref{eq:q-r-rel}) and the assumption $\pinc<2/5$ imply 
  \begin{align*}
    0 < q < r_0 < \frac{1+\delta}{\alpha_0\pinc}.
  \end{align*}
  For $j\geq 1$, we therefore have $1/j>0$ and 
  \begin{align*}
    \theta_2(j)
     = 1-q^{1/j}
     \geq 1-\left(\frac{1+\delta}{\alpha_0\pinc}\right)^{1/j}
     \geq A\eta(j).
  \end{align*}
  For $j=0,$ the definition of $\eta(0)$, statement
  (\ref{item:r1}) for the case $j=1$ shown above, and the definition of
  $\theta_2(0)$ give
  \begin{align*}
    A\eta(0) & = \eta(1)
              \leq \frac{\theta_2(1)}{A}
              = \frac{b\theta_2(0)}{A}
              < \theta_2(0).
  \end{align*}
  
  Statement (\ref{item:r2}) follows from the definition of $\theta_1(j)$.

  We now show (\ref{item:r3}). The statement is true by definition for
  $j=0$, so assume that $j\geq 1$. We first derive an upper bound on the parameter $b$ in
  terms of the constant $q$.  In particular, the constraint on $b$ from Theorem \ref{thm:k-lo}
  and the relationship $0<q<r_0$ from (\ref{eq:q-r-rel}) give
  \begin{align}
    b & \leq\frac{1}{1 + \sqrt{r_0}} 
       < \frac{1}{1 + \sqrt{q}} 
      = \frac{1 - \sqrt{q}}{1 - q}.\label{eq:b-bound-3}
  \end{align}
  The right hand side of (\ref{eq:b-bound-3}) can be further bounded
  by observing that the function $g(j):=\frac{1 - q^{1/(j+1)}}{1 -
    q^{1/j}}$ with $q>0$ increases monotonically with respect to $j$. Thus,
  for all $j\in\mathbb{N}$, we have
  \begin{align}
    \label{eq:b-bound-4}
    b < \frac{1 - \sqrt{q}}{1 - q}
    &\leq\frac{1 - q^{1/(j+1)}}{1 - q^{1/j}}.
  \end{align}
  This upper bound on parameter $b$ now immediately leads to the desired result
  \begin{align}
    b\theta_2(j)
                    & = b(1-q^{1/j})
                     \leq 1-q^{1/(j+1)}
                     = \theta_2(j+1).
  \end{align}

  Statement (\ref{item:r3b}) follows by applying the previous three
  statements in the order
  (\ref{item:r2}), (\ref{item:r1}), and (\ref{item:r3})
  \begin{align*}
    A\theta_1(j) \leq Ab\eta(j)
               \leq b\theta_2(j)
               \leq \theta_2(j+1).
  \end{align*}
  
  Next we prove statement (\ref{item:r4}). The statement is trivially
  true for $j=0$, because $r(0,A\chi)=1$, so assume that $j\geq 1$.
  By the assumption
  $\chi / n \leq \eta(j)$ and the definition of $\eta(j)$, 
  \begin{align*}
    r(j, A\chi) & = \left(1-\frac{A\chi}{n}\right)^j
     \geq \left(1-A\eta(j)\right)^j\\
    &\geq \left(1 - \left(1 - \left(\frac{1+\delta}{\alpha_0\pinc}\right)^{1/j}\right)\right)^j
      = \frac{1+\delta}{\alpha_0\pinc}.
  \end{align*}

  Finally, we prove statement (\ref{item:r5}). Again, the statement is
  trivially true for $j=0$, because $r(0,b\chi)=1$, so assume that
  $j\geq 1$.
  We derive an alternative upper
   bound on parameter $b$ in terms of $r_0$ and $q$. By the constraint on
   $b$ in Theorem \ref{thm:k-lo},
  \begin{align}
    b
    & \leq \frac{1}{1+\sqrt{r_0}} 
     = \frac{\ln(r_0)}{\ln(r_0)+\sqrt{r_0}\ln(r_0)} \nonumber\\
    & = \frac{\ln(r_0)}{\ln\left(r_0r_0^{\sqrt{r_0}}\right)} 
      = \frac{\ln r_0}{\ln q}. \label{eq:b-bound-0}
  \end{align}

  Furthermore, note that the function
  $h(j):=\frac{1 - r_0^{1/j}}{1 - q^{1/j}}$ 
  decreases
  monotonically with respect
  to $j$ when $r_0 > q > 0$, and has the limit
  $\lim_{j\rightarrow\infty}h(j)= \ln(r_0) / \ln(q)$. 
  Using (\ref{eq:b-bound-0}), it therefore
  holds for all $j\in\N$ that
  \begin{equation}
    \label{eq:b-bound}
    b \leq
    \frac{\ln r_0}{\ln q}
    \leq\frac{1 - r_0^{1/j}}{1 - q^{1/j}}.
  \end{equation}

  The assumption
  $\chi / n \leq \theta_2(j)$, the definition of $\theta_2(j)$, and
  (\ref{eq:b-bound}) now give
  \begin{align*}
    r(j, b\chi) & = \left(1 - \frac{b\chi}{n}\right)^j 
                \geq \left(1 - b\theta_2(j)\right)^j\\
               &=  \left(1 - b\left(1 -q^{1/j}\right)\right)^j
                \geq \left(1 - \left(1 - r_0^{1/j}\right)\right)^j
                = r_0,
  \end{align*}
  which completes the proof of statement (\ref{item:r5}).  
\end{proof}

\todo{}
Using Lemma \ref{lemma:r-bounds}, we are now in a position to prove that
the levels satisfy condition (G2). If the mutation rate $\chi/n$
is above the intermediary value $\eta(j)$, there is a sufficiently
high probability of reducing the mutation rate while maintaining the
fitness. Conversely, if the mutation rate is below the intermediary
value $\eta(j)$, there is a sufficiently high probability of
increasing the mutation rate while maintaining the fitness.
  \begin{lemma}
    \label{lemma:can-stay}
    Assume that the parameters $A, b,$ and $\pinc$ satisfy the
    constraints in Theorem~\ref{thm:k-lo}. Then there exists a
    constant $\delta\in(0,1/10)$ such that for all $j\in[0..k-1]$ and
    all $\ell\in[d_j]$, if Algorithm~\ref{algo:stepwise-sa} in step
    \ref{algo:select} selects a parent $(x, \chi/n)\in A_{(j, \ell)}$,
    then the offspring $(x', \chi'/n)$ created in steps
    \ref{algo:adapt} and \ref{algo:mutate} of the algorithm satisfies
    \[\prob\left((x', \chi'/n)\in \A{(j, \ell)}\right)\geq \frac{1 + \delta}{\alpha_0}.\]
  \end{lemma}

  \begin{proof}
    We will prove the stronger statement that with probability
    $(1+\delta)/\alpha_0$, we have simultaneously 
    \begin{align}
      x'\in A_{\geq j}\quad \text{and}\quad \min\left\{\frac{\chi}{n},\theta_1(j)\right\}\leq \frac{\chi'}{n}\leq \theta_2(j).\label{eq:lemma-can-stay-event}
    \end{align}
    The event (\ref{eq:lemma-can-stay-event}) is a subset of
    the event $(x',\chi'/n)\in\A{(j, \ell)}$, because a lower level $A_{(j,\ell)}$ may
    contain search points $(x',\chi'/n)$ with mutation rates
    $\chi'/n<\min(\chi/n,\theta_1(j))$.
    
    By Definition~\ref{def:sublevels}, the parent satisfies
    $x\in A_{\geq j}$ and $ \chi / n \leq \theta_2(j)$. We
    distinguish between two cases.
    
    \underline{Case 1: $\chi / n \leq \eta(j)$.}
    By Lemma~\ref{lemma:r-bounds} (\ref{item:r0}), and the monotonicity
    of $\eta$, we have $\eta(j)<1/2$.
    Note that in this case, it is still ``safe'' to increase the
    mutation rate. 
    For a lower bound, we therefore pessimistically only account for offspring where the
    mutation parameter is increased from $\chi<n/2$ to $\min(A\chi,n/2)$.
    Note first that since $A>1,$ we have
    \begin{align*}
      \frac{\chi'}{n}=\frac{\min(A\chi,n/2)}{n}>\frac{\chi}{n}.
    \end{align*}
    Also, Lemma \ref{lemma:r-bounds} (\ref{item:r1}) implies the upper bound
    \begin{align*}
      \frac{\chi'}{n}\leq \frac{A\chi}{n} \leq A\eta(j) \leq \theta_2(j).
    \end{align*}

    To lower bound the probability that $x'\in A_{\geq j}$, we
    consider the event where the mutation rate is increased, and the
    event that none
    of the first $j$ bits in the offspring are mutated with the new
    mutation parameter $\min(A\chi,n/2)$. By
    definition of the algorithm and
    using Lemma \ref{lemma:r-bounds} (\ref{item:r4}), these two events occur with probability 
    at least
    \begin{equation}
      \label{eq:inc-2}
      \pinc r(j, A\chi) \geq (1 + \delta)/\alpha_0.
    \end{equation}

    \underline{Case 2: $\eta(j) < \chi / n \leq \theta_2(j)$.}
    Note that in this case, it may be ``unsafe'' to increase the
    mutation rate. 
    For a lower bound, we pessimistically only consider
    mutation events where the mutation parameter is decreased
    from $\chi$ to $b\chi$.
    Analogously to above, since $b<1$, we have
    \begin{align}
      \frac{\chi'}{n}=\frac{b\chi}{n} < \frac{\chi}{n} \leq \theta_2(j).
    \end{align}
    Furthermore, Lemma~\ref{lemma:r-bounds} (\ref{item:r2}) implies
    the lower bound
    \begin{align}
      \frac{\chi'}{n}=\frac{b\chi}{n} > \frac{b\eta(j)}{n} \geq \theta_1(j).
    \end{align}
    
    To lower bound the probability that  $x'\in A_{\geq j}$, we
    consider the event where the mutation parameter is decreased from
    $\chi$ to $b\chi$, and the offspring $x'$ is not downgraded to a
    lower level.
    By definition of the algorithm, $r(j, b\chi)$, and using
    Lemma~\ref{lemma:r-bounds} (\ref{item:r5}), these two events occur with
    probability
    \begin{equation}
      (1-\pinc) r(j,b\chi) \geq (1+\delta)/\alpha_0.
    \end{equation}

    Hence, we have shown that in both cases, the event in Eq.
    (\ref{eq:lemma-can-stay-event}) occurs with probability at least
    $(1+\delta)/\alpha_0$, which completes the proof.
  \end{proof}

  We now show that the edge levels satisfy condition (G1) of the
  level-based theorem. As we
  will show later, the upgrade probability for non-edge levels is
  constant.

  \begin{lemma}
    \label{lemma:upgrade-prob}
    Assume that the parameters $b$ and $\pinc$ satisfy the constraints
    in Theorem \ref{thm:k-lo}. Then for any $j\in[0..k-1]$, and any
    search point $(x, \chi/n)\in A_{(j, d_j)}$ selected in step
    \ref{algo:select} of Algorithm \ref{algo:stepwise-sa} applied to
    \leadingones, the offspring
    $(x', \chi'/n)$ created in steps \ref{algo:adapt} and
    \ref{algo:mutate} satisfies
    $
      \prob\left((x', \chi'/n)\in \A{(j+1, 1)}\right)= \Omega(1/j).
    $
  \end{lemma}
  \begin{proof}
    By the definition of level $A_{(j, d_j)}$, we
    have $\theta_1(j)\leq \chi/n \leq \theta_2(j)$ and so by Lemma~\ref{lemma:r-bounds} (\ref{item:r3}),
    we have $b\chi/n \leq \theta_2(j+1)$. 
    Given the definition of levels $\A{(j+1, 1)}$, it suffices for a
    lower bound to only consider the probability
    of producing an offspring $(x',\chi'/n)$ with lowered mutation rate
    $\chi'/n=b\chi/n\leq \theta_2(j+1)$ and fitness
    $\textsc{LO}_k(x')\geq j+1$.

    We claim that if the mutation rate is lowered, 
    the offspring has fitness $\textsc{LO}_k(x')\geq j+1$ 
    with probability $\Omega(1/j).$  Since the parent belongs to level
    $A_{(j,d_j)}$, it has fitness $\textsc{LO}_k(x)\geq j$, so we need
    to estimate the probability of not flipping the first $j$ bits,
    and obtain a 1-bit in position $j+1$.

    We now estimate the probability of obtaining a 1-bit in position
    $j+1$, assuming that the parent $x$ already has a 1-bit in this
    position, for any $j\in[0..k-1]$.
    Using that $\theta_2(j)$ decreases monotonically in $j$,
    the definition of $\theta_2(0)$, and the lower bound on the
    parameter $\zeta > 23 / 144$ from
    Eq. (\ref{eq:zeta-lower-bound}),
    the probability of not mutating
    bit-position $j+1$ with the lowered mutation rate $b\chi / n$ is 
    \begin{align*}
      1-\frac{b\chi}{n}
       &\geq 1-b\theta_2(j)
       \geq 1-b\theta_2(0)\\
       &= 1 - \theta_2(1)
       = 1 - \frac{1-\zeta}{\alpha_0}
       > 1 - \frac{1-\frac{23}{144}}{\alpha_0}
       = \Omega(1).
    \end{align*}

    If the parent $x$ does not have a 1-bit in position $j+1$, we need
    to flip this bit-position. By the
    definition of $\theta_1(j)$ in Eq. (\ref{eq:theta1-def}), the
    probability of this event is in the case $j\geq 1$ 
    \begin{align}
      \frac{b\chi}{n}
      & \geq b\theta_1(j)\label{eq:bchi-bound-start}
        = \frac{b^2}{2A}\left(1 - \left(\frac{1+\delta}{\alpha_0\pinc}\right)^{1/j}\right)\\
      & \geq \frac{b^2}{2Aj} \ln \left(\frac{\alpha_0\pinc}{1+\delta}\right)
        = \Omega(1/j),\label{eq:bchi-bound}
    \end{align}
    where the last inequality follows from
    Lemma~\ref{lemma:inv-bound}. If $j=0$, we use that $\theta_1(j)$
    decreases monotonically in $j$ and Eqs. (\ref{eq:bchi-bound-start})--(\ref{eq:bchi-bound}) to
    show that the probability of flipping bit $j+1=1$ is
    \begin{align*}
      \frac{b\chi}{n} \geq b\theta_1(0) > b\theta_1(1) = \Omega(1).
    \end{align*}
    The claim that we obtain a 1-bit in position $j+1$ with probability
    $\Omega(1/j)$ is therefore true.
    
    Thus, the probability of lowering the mutation rate to $b\chi / n$, 
    obtaining a 1-bit in position $j+1$, and not flipping the first
    $j$ positions is, using the definition of $\theta_2(j)$ in
    (\ref{eq:theta2-def}),
    \begin{multline*}
      (1-\pinc) \Omega(1/j)\left(1 - \frac{b\chi}{n}\right)^j\\
       > \Omega(1/j) \left(1 - \theta_2(j)\right)^j
       = \Omega(1/j) \left(\frac{1 - \zeta}{\alpha_0}\right)
      = \Omega(1/j),
    \end{multline*}
    which completes the proof.
  \end{proof}

  \subsection{Individuals with Too High Mutation Rates}
  \label{sec:control-bad}

  Highly fit individuals with incorrect parameter settings can cause
  problems for self-adaptive EAs. If there are too many such ``bad''
  individuals in the population, they may dominate the population,
  propagate bad parameter settings, and thus impede progress. In this
  section, we therefore bound the number of such bad individuals.
  We define a region $B\subset\extsspace$ containing search points
  with a mutation rate that is too high relative to their fitness. For the
  constant $\zeta\in(0,1)$ defined in Eq. (\ref{eq:zeta-def}), let
  \begin{align}
    \label{eq:B-def}
    B &:= \Big\{(x, \chi/n)\in A_j\times [\ep,1/2] \mid \\
    & \quad j\in\mathbb{N}_0
    \;\wedge\; \forall y\in\mathcal{X}
      \prob_{x'\sim\pmut(y, \chi)}(x'\in A_{\geq j})<\frac{1-\zeta}{\alpha_0}
    \Big\}.\nonumber
  \end{align}
  Note that by the definition of the function $\theta_2(j)$, it holds
  for all $y\in\sspace$ that the statement
  $\prob_{x'\sim\pmut(y, \chi)}(x'\in A_{\geq j}) = (1 - \chi/n)^j <
  (1 - \zeta)/\alpha_0$ is analogous to $\chi/n >
  \theta_2(j)$. Therefore, the region $B$ can also be expressed as
  \begin{align}
    \label{eq:B-alt}
    B = \bigcup_{j=0}^{k-1} A_{>j}\times \left(\min \left(\frac{1}{2},\theta_2(j+1)\right), \min\left(\frac{1}{2},\theta_2(j)\right)\right].
  \end{align}
  
  An individual $(x, \chi/n)\in B$ is said to have too high mutation
  rate. To see why, recall that the number of offspring of
  $(x, \chi/n)$ is never more than $\alpha_0$. Since the probability
  an offspring has fitness as least as good as $f_k(x)$ is less than
  $1/\alpha_0$, in expectation less than one offspring maintains the
  fitness, making it unlikely a lineage of $(x, \chi/n)$ will be able
  to make progress towards the optimum. This corresponds to the
  ``error threshold'' discussed in \cite{bib:Lehre2010}; by Corollary
  1 in \cite{bib:Lehre2010}, the probability that a group of
  individuals of size $\text{poly}(k)$ staying in $B$ will optimise
  $\textsc{LO}_k$ in sub-exponential time is $e^{-\Omega(k)}$.

  The levels given by Definition \ref{def:sublevels} are not disjoint
  from the region $B$ defined above. This is an important departure
  from the approach used in \cite{bib:Dang2016}, where $B$ is
  effectively removed from the search space and the level-based
  theorem is applied to a partition over $\mathcal{Y}\setminus
  B$. This is not effective in our setting, since an individual can
  have a sudden increase in fitness but can not significantly decrease
  mutation rate. Such an individual may have too high mutation rate
  with respect to its new fitness, but we still depend on this
  individual having correctly tuned mutation with respect to the old
  fitness value. Therefore, an individual in some $A_{(j, \ell)}$ may
  mutate in and out of $B$ before its mutation rate has been adapted
  to maintain a fitness higher than $j$. While individuals may
  occasionally jump into the bad region, if too many individuals are
  in $B$ at a given time this may destroy the progress of the
  algorithm. In particular, if $|P_t \cap B| > \mu$, then assuming all
  individuals in $B$ have strictly better fitness and higher mutation
  rate than those not in $B$, only individuals from $B$ will be
  selected for mutation and the next generation will consist only of
  individuals with very high mutation rate. Therefore, it is critical
  that for any generation $t\in\N$, the number of individuals in $B$
  will be less than $\mu$ with overwhelmingly high probability. We
  prove this with the following lemma, which is similar to Lemma 2
  from \cite{bib:Dang2016}, but with the notable difference that the
  size of $B$ can be controlled within a single generation, regardless
  the configuration of $P_{t-1}$.

  \begin{lemma}\label{lemma:few-bad}
    Let $B\subseteq\extsspace$ be as defined in Eq. (\ref{eq:B-def}) for
    a constant $\zeta\in(0,1)$. Then for any generation $t\in\mathbb{N}$ of
    Algorithm \ref{algo:stepwise-sa} applied to \leadingones,
    \[\prob(|B \cap P_t|\geq(1-\zeta/2)\mu)\leq e^{-\Omega(\mu)}.\]
  \end{lemma}

  \begin{proof}
        Consider some parent $(x,\chi/n)$ selected in generation $t-1\geq 0$
    and step \ref{algo:select} of Algorithm~\ref{algo:stepwise-sa}.
    Referring to steps \ref{algo:adapt} and \ref{algo:mutate}, first a
    new mutation parameter $\chi'$ is chosen, then a new bitstring
    $x'$ is obtained from $x$ using bitwise mutation with mutation
    parameter $\chi'$. To obtain an upper bound on the probability
    that $(x', \chi'/n)$ is in $B$, we proceed in cases based on the
    outcome of sampling $\chi'$, namely, whether the chromosome
    $(x, \chi'/n)$ is in $B$.

    \underline{If $(x, \chi'/n)\in B$:} then it follows immediately
    from the definition of $B$ that independently of the chosen parent
    $x$, it holds
    \begin{align}
      \label{eq:B-bound-1}
      \probb{(x',\chi'/n)\in B}
      < \frac{1 - \zeta}{\alpha_0}.
    \end{align}
    
    \underline{If $(x, \chi'/n)\not\in B$:} then for $(x',
    \chi'/n)$ to end up in
    $B$, by Eq. \eqref{eq:B-alt} it is necessary that $x'\in A_{\geq
      u}$ for some $u > j$, where $x\in A_j$ and $r(u, \chi') < (1 -
    \zeta)/\alpha_0$. Since $\chi'/n <
    1/2$, the probability of obtaining $x'\in A_{\geq
      u}$ is no more than
    \begin{align}
      \label{eq:B-bound-2}
      \left(1 - \frac{\chi'}{n}\right)^{u-1}\left(\frac{\chi'}{n}\right)
      < \left(1 - \frac{\chi'}{n}\right)^{u}
      < \frac{1 - \zeta}{\alpha_0}.                                        
    \end{align}

    Since each of the $\lambda$ individuals in population
    $P_t$ are sampled independently and identically,
    Eqs. \eqref{eq:B-bound-1} and \eqref{eq:B-bound-2} imply $|B\cap
    P_t|$ is stochastically dominated by a binomially distributed
    random variable
    $Z\sim\bin(\lambda,\zz)$ which has expectation
    $\mu(1-\zeta)$.  By a Chernoff bound,
    \begin{align*}
      \probb{|B \cap P_t|\geq \mu(1-\zeta/2)}
       \leq  \probb{Z\geq \mu(1-\zeta/2)}\\
       = \probb{Z\geq \expect{Z}\left(1+\tfrac{1}{2(1-\zeta)}\right)} 
        = e^{-\Omega(\mu)}.
    \end{align*}
  \end{proof}

  \subsection{Applying the Level-based Theorem}
  \label{sec:apply-level-based}

  We now combine the results of Sections \ref{sec:def-sublevels},
  \ref{sec:calc-probs}, and \ref{sec:control-bad} to prove
  Theorem~\ref{thm:k-lo} using Theorem \ref{thm:level-based}.

  \begin{proof}[Proof (of Theorem \ref{thm:k-lo})]
    We partition the search space $\extsspace$ into the sets
    $A_{(j, \ell)}$ from Definition \ref{def:sublevels}, where
    $j\in[0..k-1]$ and $\ell\in[d_j]$, along with $A_{(k,1)}$, and
    define $\A{(j,\ell)}$ as in (\ref{eq:Aorbetter-def}).

    We say that a generation $t$ is ``failed'' if the population $P_t$
    contains more than $(1-\zeta/2)\mu$ individuals in region $B$.
    First, we will optimistically assume that no generations
    fail. Under this assumption, we will prove that the conditions of
    Theorem \ref{thm:level-based} hold, leading to an upper bound on
    the expected number of function evaluations $t_0(k)$ until a
    search point in $A_{(k, 1)}$ is created (i.e. a global optimum is
    found). Then in the end we will use a restart argument to account
    for failed generations.

    Let $\gamma_0:=(\zeta/2)(\mu/\lambda)$. In the following arguments,
    we will make consistent use of the important fact that for
    $\gamma\in(0, \gamma_0]$, if there are $\gamma\lambda$ individuals
    in levels $\A{(j, \ell)}$ for some $j\in\{0,\dots,k-1\}$ and $\ell\in[d_j]$,
    then the probability of selecting an individual from $\A{(j, \ell)}$
    is $\gamma\alpha_0$. To see this, we note that individuals in
    $\A{(j, \ell)}$ are guaranteed to be ranked above those in
    $\extsspace\setminus\left(\A{(j, \ell)} \cup B\right)$ in line
    \ref{algo:order} of Algorithm \ref{algo:stepwise-sa}, since
    individuals in
    $\extsspace\setminus\left(\A{(j, \ell)} \cup B\right)$ must have
    fitness either strictly less than $j$, or equal to $j$, and hence
    have mutation rate too low to be contained in $\A{(j,
      \ell)}$. Recalling that $|P_t\cap B|\leq (1-\zeta/2)\mu$, it
    follows that all $\gamma\lambda\leq (\zeta/2)\mu$ individuals of
    $\A{(j, \ell)}$ are among the $\mu$ fittest in the
    population. Therefore, the probability of selecting an individual
    from $\A{(j, \ell)}$ indeed is
    $\gamma\lambda / \mu  = \gamma\alpha_0$.

    We now show that conditions (G1) and (G2) of Theorem
    \ref{thm:level-based} hold for each level $A_{(j,\ell)}$ where
    $j\in[0..k-1]$ and $\ell\in[d_j]$. We assume that the current
    population has at least $\gamma_0\lambda$ individuals in levels
    $\A{(j, \ell)}$. We distinguish between the case $\ell < d_j$,
    i.e., when it suffices to increase the mutation rate to upgrade to
    the next level, and the case $\ell = d_j$, i.e., when it may be
    necessary to increase the fitness to reach the next level.

    \underline{$\ell < d_j$:} To verify condition (G2), we must
    estimate the probability of producing an offspring in levels
    $\A{(j, \ell+1)}$, assuming that there are at least
    $\gamma\lambda$ individuals in levels $\A{(j, \ell+1)}$, for any
    $\gamma\in(0, \gamma_0]$.  To produce an offspring in levels
    $\A{(j, \ell+1)}$, it suffices to first select a parent
    $(x, \chi/n)$ from $\A{(j, \ell+1)}$, and secondly create an
    offspring $(x', \chi'/n)$ in levels $\A{(j, \ell+1)}$. The
    probability of selecting such a parent is at least
    $\gamma\alpha_0$. Assuming that the parent is in level
    $A_{(u,v)}\subseteq \A{(j, \ell+1)}$, and applying Lemma
    \ref{lemma:can-stay} to level $A_{(u, v)}$, the probability that
    the offspring $(x',\chi'/n)$ is in levels
    $\A{(u, v)}\subseteq \A{(j, \ell+1)}$ is $(1+\delta) / \alpha_0$
    for some $\delta\in(0,1)$.  Thus the probability of selecting a
    parent in levels $\A{(j, \ell+1)}$, then producing an offspring in
    levels $\A{(j, \ell+1)}$, is at least
    $\gamma\alpha_0\left(\frac{1+\delta}{\alpha_0}\right) =
    \gamma(1+\delta)$, so condition (G2) is satisfied.

    To verify condition (G1), we estimate the probability of producing
    an offspring in levels $\A{(j, \ell+1)}$. If the parent is in
    levels $\A{(j, \ell+1)}$, then again by Lemma
    \ref{lemma:can-stay}, the offspring is in levels $\A{(j, \ell+1)}$
    with probability at least $(1+\delta) /\alpha_0$.

    On the other hand, if the parent $(x, \chi/n)$ is in level
    $A_{(j, \ell)}$, then we consider the probability of producing an
    offspring $(x', \chi'/n)\in\A{(j, \ell+1)}$ by increasing the
    mutation rate from $\chi$ to $A\chi$, and maintaining the fitness
    $x'\in\A{j}$. By assumption, $\ell<d_j$, so the level-definition
    implies that the parent has mutation rate
    $\chi / n < \theta_1(j)< \eta(j)$. Hence, by
    Lemma~\ref{lemma:r-bounds} (\ref{item:r4}), the probability of
    increasing the mutation parameter to $A\chi$ and maintaining at
    least $j$ leading one-bits is at least
    $\pinc r(j,A\chi) \geq (1+\delta)/\alpha_0.$

    Taking into account that the probability of selecting a parent 
    in $\A{(j, \ell)}$ is at least $\alpha_0\gamma_0$, the probability
    of producing an offspring in $\A{(j, \ell+1)}$ is at least 
    \begin{align}
      \label{eq:low-z-def}
      \gamma_0\alpha_0 \left(\frac{1+\delta}{\alpha_0}\right)
      & = (1+\delta)\gamma_0 =: z_{(j, \ell)}.
    \end{align}

    \underline{$\ell = d_j$:} To show (G2) we assume that there are at
    least $\gamma\lambda$ individuals in levels $\A{(j+1, 1)}$, for
    $\gamma\in(0, \gamma_0]$. We again apply Lemma
    \ref{lemma:can-stay} to show the probability of selecting an
    individual from $\A{(j+1, 1)}$ and producing a new individual also
    in $\A{(j+1, 1)}$ is at least $\gamma(1+\delta)$, showing (G2) is
    satisfied.

    For condition (G1), we only consider parents selected from levels
    $\A{(j+1, 1)}$.  If the parent $(x, \chi/n)$ is in $A_{(j, d_j)}$,
    then by Lemma~\ref{lemma:upgrade-prob} the offspring
    $(x',\chi'/n)$ is in levels $\A{(j+1, 1)}$ with probability at
    least $\Omega(1/j)$. Otherwise, if the parent is already in
    levels $\A{(j+1, 1)}$, then by Lemma \ref{lemma:can-stay}, the
    offspring is in levels $\A{(j+1, 1)}$ with probability at least
    $(1 + \delta) / \alpha_0=\Omega(1)$. In both cases, the
    probability of selecting a parent from $\A{(j, d_j)}$ and
    producing an offspring in levels $\A{(j+1, 1)}$ is at least
    \begin{align}
      \label{eq:edge-z-def}
      \gamma_0\alpha_0 \Omega(1/j) = \Omega(1/j) =: z_{(j,d_j)}.
    \end{align}

    To verify that $\lambda\geq c\ln(n)$ is large enough to satisfy
    condition (G3), we first must calculate $m$, the total number of
    sub-levels. Referring to Definition \ref{eq:depth-def} and using
    $\theta_1(j) < 1/2$, the depth of each level $j$ is no more than
    $d_j< \lceil\log_A\left(\frac{1}{2\ep}\right)\rceil= O(\log(n))$
    for all $j\in\{0,\dots,k-1\}$. Therefore, $m = O(k\log(n))$ and so
    $\lambda \geq c\ln(n)$ satisfies (G3) for $c > 1$ large enough.
    Thus we have found parameters
    $z_{(0, 1)}, z_{(0, 2)}, \dots, z_{(k-1, d_{k-1})}, \delta,$ and
    $\gamma_0,$ such that all three conditions of
    Theorem~\ref{thm:level-based} are satisfied. Assuming no failure,
    the expected time to reach the last level is no more than
    \begin{align*}
      t_0(k) & \leq \left(\frac{8}{\delta^2}\right)
               \sum_{j=0}^{k-1}\sum_{\ell=1}^{d_j}
               \left(
               \lambda\log
               \left(\frac{6\delta\lambda}{4+z_{(j,\ell)}\delta\lambda}\right)
               +\frac{1}{z_{(j,\ell)}}
               \right)\\
             & = \sum_{j=0}^{k-1}\sum_{\ell=1}^{d_j-1} O\left(\lambda\log\left(\frac{1}{z_{(j,\ell)}}\right)+\frac{1}{z_{(j,\ell)}}\right)\\
               &\quad + \sum_{j=0}^{k-1} O\left(\lambda\log(\lambda)+\frac{1}{z_{(j,d_j)}}\right)\\
             & = O(k\lambda\log(n) + k\lambda\log(\lambda) + k^2),
    \end{align*}
    using that $z_{(j,\ell)}=\Omega(1)$ for all $\ell<d_j$,
    and $z_{(j,d_j)}=\Omega(1/j)$.

    Finally, we account for ``failed'' generations where our
    assumption that there are less than $(1-\zeta/2)\mu$ individuals
    in region $B$ does not hold. We refer to a sequence of
    $2t_0(k) / \lambda$ generations as a \textit{phase}, and call a
    phase \emph{good} if for $2t_0(k)/\lambda$ consecutive generations
    there are fewer than $(1-\zeta/2)\mu$ individuals in $B$. By
    Lemma~\ref{lemma:few-bad} and a union bound, a phase is good with
    probability $1-2t_0(k)/\lambda e^{-\Omega(\mu)} = \Omega(1)$, for
    $\mu = \Omega(\log(n))$. By Markov's inequality, the probability
    of reaching a global optimum in a good phase is at least
    $1/2$. Hence, the expected number of phases required, each costing
    $2t_0(k)$ function evaluations, is $O(1)$.
  \end{proof}
  
  \section{Experiments}
  \label{sec:experiments}

  The theoretical analysis of the $(\mu, \lambda)$ self-adaptive EA is
  complemented by some experiments on a wider variety of problems. In
  addition to the standard \textsc{OneMax} function, we consider
  \begin{align*}
      \textsc{Jump}_k(x) &:=
    \begin{cases}
      \textsc{OM}(x) + k & \text{if } \textsc{OM}(x) < n - k\\
      \textsc{OM}(x) - k & \text{if } n - k \leq \textsc{OM}(x) < n\\
      n +1 &\text{if } \textsc{OM}(x) = n.
    \end{cases}\\
    \textsc{SubString}_k(x) &:= \max_{1\leq i\leq n} i \cdot
                              \textstyle\prod_{j=\max\{i-k+1,
                              1\}}^ix_j.
  \end{align*}
  The $\textsc{SubString}_k$ function is similar to the function in
  \cite{bib:Chen2009} of the same name. The value of
  $\textsc{SubString}_k$ is the maximal position of the
  substring $1^k$, if such a substring exists, otherwise it
  is just the number of leading 1-bits. While the function in
  \cite{bib:Chen2009} has a unique global optimum at the point $1^n$,
  all strings of the form $\{0,1\}^{n-k}1^k$ are optimal for our
  $\textsc{SubString}_k$ function.
  
  In a first set of experiments, we examined how Algorithm
  \ref{algo:stepwise-sa} adjusts mutation rates relative to fitness on
  several contrasting fitness landscapes. In each run, we chose the
  parameter settings $\lambda = 8\ln(n)$, $\mu = \lambda / 15$,
  $A = 1.5$, $b=0.7$, and $\pinc = 0.25$, where $\lambda$ and $\mu$
  are rounded to the nearest integer. For the functions
  \textsc{LeadingOnes}, $\textsc{SubString}_{\sqrt{n}}$,
  \textsc{OneMax}, and $\textsc{Jump}_3$, we recorded the fitness and
  mutation rate $\chi/n$ of the top-ranked individual in each
  generation. In a second set of experiments, we compared Algorithm
  \ref{algo:stepwise-sa} to other algorithms on
  $\textsc{LeadingOnes}_k$, $\textsc{SubString}_k$, and a version of
  $\textsc{OneMax}$ in which only an unknown selection of $k$ bits
  contribute to fitness ($\textsc{OneMax}_k$). For these experiments,
  we chose the parameters $\lambda=16\ln(n),$ $\mu= \lambda / 8,$
  $A=1.2,$ $b=0.7,$ $\pinc=0.25$ for Algorithm
  \ref{algo:stepwise-sa}. The change in parameter settings was not
  particularly motivated, although note that both respect the
  conditions imposed by Theorems \ref{thm:k-lo}, since $\pinc = 0.25$
  satisfies $1/16 < 1/15 < 1/4 < 2/5$, and $b = 0.7$ satisfies
  $7/10 < 1/(1 + \sqrt{4/45}) \approx 0.78$.

  The results from the first set of experiments are summarised in
  Fig. \ref{fig:mutvsfitness}. We set $n = 500$ for
  \textsc{LeadingOnes}, $\textsc{SubString}_{\sqrt{n}}$, and
  \textsc{OneMax}, while for $\textsc{Jump}_3$ we set $n = 100$, and
  performed 100 trials for each function. At the beginning of a trial,
  all individuals were given a starting mutation strength of
  $\chi = 1$. For each function, we plotted the median mutation rate
  $\chi/n$ per fitness value in blue, with the 95th percentile shaded
  in grey. Finally, to aid interpretation we plotted in red the
  ``error threshold'', i.e. the value of $\chi/n$ such that the
  expected number of offspring with fitness at least as good as the
  parent's is only 1 \cite{bib:Lehre2010}. For \textsc{LeadingOnes},
  the error threshold is thus approximately the $\theta_2$ function
  introduced in Section \ref{sec:def-sublevels}.

  Fig. \ref{fig:mutvsfitness} shows that Algorithm
  \ref{algo:stepwise-sa} tuned the mutation rate of the top-ranked
  individual very differently depending on the fitness landscape. For
  \textsc{LeadingOnes}, we see the top individual's mutation rate
  quickly rose to a small factor below $\theta_2$, then gradually
  lowered mutation rate as fitness increased. This supports our
  theoretical analysis of $\textsc{LeadingOnes}_k$, in which we argued
  that the mutation rate rises to an ``edge region'' comprising of
  mutation rates just below the error threshold. We found similar
  behaviour for $\textsc{SubString}_{\sqrt{n}}$, where again the
  algorithm quickly rose to a close approximation below the error
  threshold. However, the results for \textsc{OneMax} and
  $\textsc{Jump}_3$ are less conclusive. First, we were unable to
  derive an exact expression for the error threshold for these
  functions, which makes the trajectory of the mutation rates more
  difficult to interpret. Instead we include in green the mutation
  rate for a single individual to maximise the expected difference in
  its fitness before and after mutation, in order to provide some
  context for interpreting the effectiveness of mutation rates. For
  \textsc{OneMax}, it is known this drift-maximising rate is
  $\Theta(1/n)$ when $\textsc{OM}(x) \geq 2n/3$ \cite{bib:Doerr2018d},
  while for $\textsc{Jump}_3$, the ideal rate is $3 / n$ for jumping
  the gap. In terms of the trajectory of mutation rates, on
  \textsc{OneMax} the algorithm correctly increased its mutation rate
  at first, but also seems to have kept mutation rate well above $1/n$
  for much of the search process. This could explain its relative
  inefficiency on $\textsc{OneMax}_k$ in the next set of experiments.
  The behaviour is similar for $\textsc{Jump}_3$, except that mutation
  rate increased toward the ideal rate while at the edge of the gap,
  and occasionally reached even higher values. The tendency for
  mutation rate to dramatically increase during lack of progress is
  reassuring, since a common difficulty in self-adaptation of mutation
  rates is that mutation rates may indefinitely decrease when it is
  difficult to increase fitness \cite{Liang2001}.
  
  \begin{figure*}[t]
    \centering
    \includegraphics[width=1.0\textwidth]{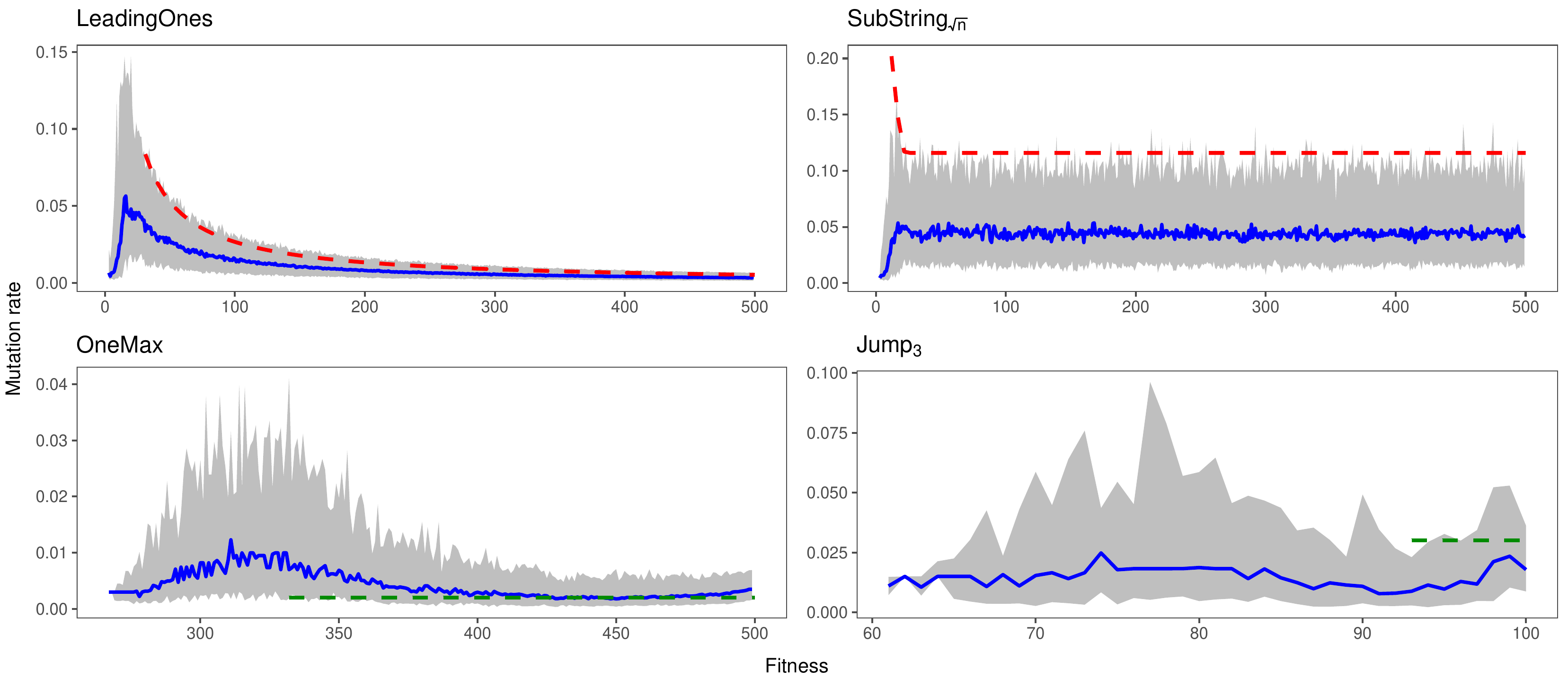}
    \caption{Fitness and mutation rate of the most fit individual per
      generation of Algorithm \ref{algo:stepwise-sa} with
      $\lambda = 8\ln(n)$, $\mu = \lambda / 15$, $A=1.5$, $b=0.7$, and
      $\pinc = 0.25$. Median mutation rate is the blue line, while
      the 95-th percentile is shaded grey. Top: the dashed red line
      indicates the error threshold, past which mutation rates will be
      ineffective for the given fitness. Bottom: the dashed green line
      shows the drift-maximising mutation rate.}
    \label{fig:mutvsfitness}
  \end{figure*}
  
  \begin{figure*}[t]
    \centering
    \includegraphics[width=\linewidth]{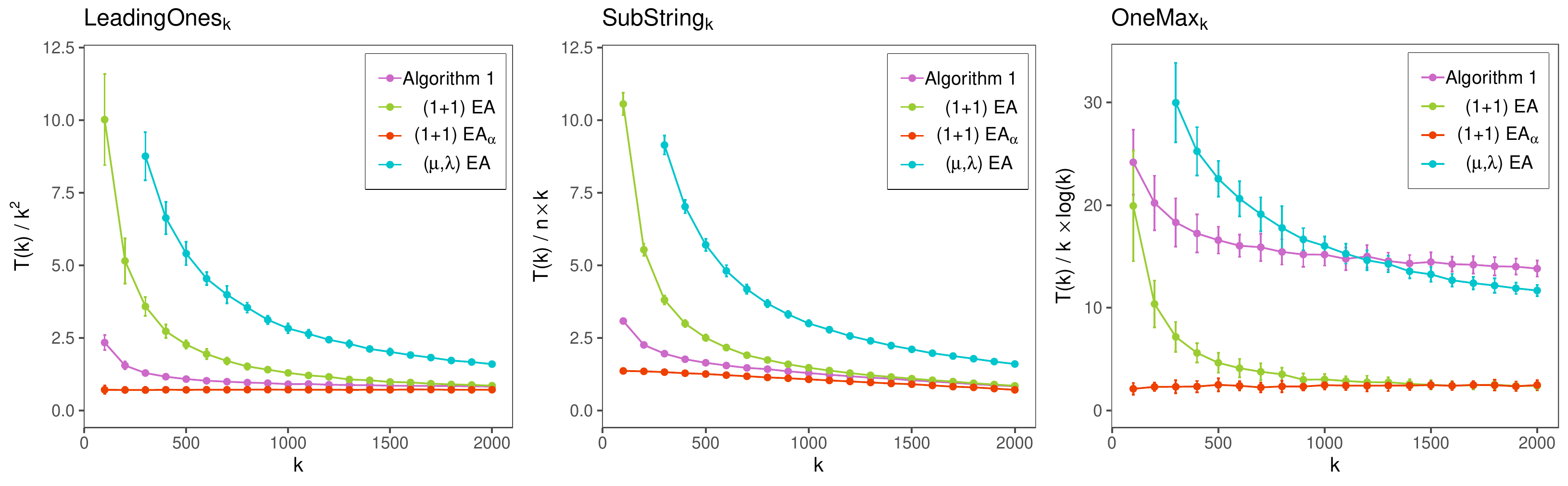}
    \caption{Runtime as a function of $k$ for fixed $n=2000$,
      normalised to show impact of adaptation. Points show the median
      runtime, with error bars extending beyond the interquartile
      range as $\pm 1.5\cdot \text{IQR}$.  Parameter settings:
      Algorithm \ref{algo:stepwise-sa} with
      $\lambda=16\ln(n), \mu=2\ln(n),$ adaptation parameters
      $A=1.2, b=0.7,\pinc=0.25$. $(1+1)$ EA with mutation rate
      $1/n$. $(1+1)$ EA$_\alpha$ with $A=1.2$ and
      $b=0.85$. $(\mu, \lambda)$ EA with
      $\lambda=16\ln(n), \mu=2\ln(n)$ mutation rate $2/(5n)$.}
    \label{fig:fixed-n}
  \end{figure*}

  In the second set of experiments, summarised in Fig.
  \ref{fig:fixed-n}, we compared the self-adaptive EA to the $(1+1)$ EA,
  the $(\mu, \lambda)$ EA, as well as to the $(1+1)_{\alpha}$ EA from
  \cite{bib:Doerr2018} with the parameter settings $A = 1.2$ and
  $b = 0.85$ (for the $(1+1)_{\alpha}$ EA). On each of the functions
  \leadingones, \substring, and $\textsc{OneMax}_k$, we tested the
  algorithms on a range of possible choices for the adversary by
  performing 100 runs of each algorithm for values of $k$ between
  $100$ and $n=2000$. The y-axes in Fig. \ref{fig:fixed-n} show the
  runtime divided by the asymptotic running time of a $(1+1)$ EA which
  knows the value $k$ beforehand. The effect of this rescaling is that
  algorithms which successfully adapt to the parameter $k$ should
  remain relatively constant along the y-axis as $k$ changes.

  On all three functions, the two adaptive algorithms had runtimes
  proportional to an EA which knew $k$ beforehand. However, while both
  also drastically outperformed the static algorithms for smaller $k$
  on \leadingones and \substring, on $\textsc{OneMax}_k$, Algorithm
  \ref{algo:stepwise-sa} performed comparably to the static $(1+1)$ EA
  only for small $k$, and did worse than the $(1+1)$ EA as $k$ grew
  larger. This is somewhat expected, since it is known that the
  $(1+1)$ EA easily outperforms many population-based algorithms on
  \textsc{OneMax}. It is possible that the benefits of adaptation will
  not overcome the penalty of maintaining a population except for much
  larger values of $n$.

\section{Conclusion}
\label{sec:conclusion}

Effective parameter control is one of the central challenges in
evolutionary computation. There is empirical evidence that
self-adaptation -- where parameters are encoded in the chromosome of
individuals -- can be a successful control mechanism in evolutionary
strategies. However, self-adaptation is rarely employed in discrete
EAs~\cite{bib:Back1992,smith_self_1996}.  The theoretical
understanding of self-adaptation is lacking.

This paper demonstrates both theoretically and empirically that
adopting a self-adaptation mechanism in a discrete, non-elitist EA can
lead to significant speedups. We analysed the expected runtime of the
$(\mu, \lambda)$ EA with self-adaptive mutation rates on \leadingones
in the context of an adversarial choice of a hidden problem parameter
$k$ that determines the problem structure. We gave parameter settings
for which the algorithm optimises $\textsc{LeadingOnes}_k$ in expected
time $O(k^2)$, which is asymptotically optimal among any unary
unbiased black box algorithm which knows the hidden value $k$.  This
is a significant speedup compared to, e.g., the (1+1) EA using any
choice of static mutation rate.  In fact, the algorithm even has an
asymptotic speedup compared with the state-of-the art parameter
control mechanism for this problem \cite{bib:Doerr2015}.  Future work
should extend the analysis to more general classes of problems, such
as linear functions and multi-modal problems. We expect that applying
the level-based theorem over a two-dimensional level-structure will
lead to further results about self-adaptive EAs.

\bibliographystyle{abbrv}

\section{Appendix}
\label{sec:appendix}

\begin{lemma}\label{lemma:inv-bound}
  For all $c>0$ and $j>0$,
    $1-c^{1/j}\geq \ln(1/c)/j.$ 
\end{lemma}
\begin{proof}
  $ -\frac{1}{j}\ln(1/c) = \ln\left(c^{1/j}\right) \leq c^{1/j}-1. $
\end{proof}

\subsection*{Acknowledgements}
The authors would like to thank Dr Duc-Cuong Dang for suggesting
the problem of determining the runtime of a self-adaptive evolutionary
algorithm on the \leadingones problem.

\end{document}